\documentclass[11pt,english,french,american]{article}
\usepackage[T1]{fontenc}
\usepackage[latin9]{inputenc}
\usepackage{array}
\usepackage{float}
\usepackage{multirow}
\usepackage{amsmath}
\usepackage{amsthm}
\usepackage{amssymb}
\usepackage{graphicx}

\makeatletter

\providecommand{\tabularnewline}{\\}
\floatstyle{ruled}
\newfloat{algorithm}{tbp}{loa}
\providecommand{\algorithmname}{Algorithm}
\floatname{algorithm}{\protect\algorithmname}

\theoremstyle{plain}
\newtheorem{thm}{\protect\theoremname}
\theoremstyle{plain}
\newtheorem{lem}[thm]{\protect\lemmaname}

\@ifundefined{date}{}{\date{}}
\usepackage{amsmath}
\usepackage{amsthm}
\usepackage{amssymb}
\usepackage{subfig}
\usepackage{color}
\usepackage[english]{babel}
\usepackage{graphicx}
\usepackage{wrapfig,epsfig}
\usepackage{epstopdf}
\usepackage{url}
\usepackage{graphicx}
\usepackage{color}
\usepackage{epstopdf}
\usepackage{scrextend}
\usepackage[T1]{fontenc}
\usepackage{bbm}
\usepackage{comment}

\usepackage{tablefootnote}
\usepackage[flushleft]{threeparttable}

\usepackage{tikz}
\usepackage{hyperref}  
\hypersetup{colorlinks=true,citecolor=blue,linkcolor=blue} 
\usetikzlibrary{arrows}
\usepackage[margin=1in]{geometry}
\usepackage{algorithmic}
\floatname{algorithm}{Algorithm}
\usepackage[titletoc,title]{appendix}

\allowdisplaybreaks

\makeatother

\usepackage{babel}
\makeatletter
\addto\extrasfrench{%
   \providecommand{\fg}{\ifdim\lastskip>\z@\unskip\fi~\frqq}%
}

\makeatother
\addto\captionsamerican{\renewcommand{\lemmaname}{Lemma}}
\addto\captionsamerican{\renewcommand{\theoremname}{Theorem}}
\addto\captionsenglish{\renewcommand{\algorithmname}{Algorithm}}
\addto\captionsenglish{\renewcommand{\lemmaname}{Lemma}}
\addto\captionsenglish{\renewcommand{\theoremname}{Theorem}}
\addto\captionsfrench{\renewcommand{\algorithmname}{Algorithme}}
\addto\captionsfrench{\renewcommand{\lemmaname}{Lemme}}
\addto\captionsfrench{\renewcommand{\theoremname}{Théorème}}
\providecommand{\lemmaname}{Lemma}
\providecommand{\theoremname}{Theorem}

\begin{document}
\global\long\def\L{\mathbb{\mathcal{L}}}%
\global\long\def\E{\mathbb{E}}%
\global\long\def\R{\mathbb{\mathbb{R}}}%
\global\long\def\P{\mathbb{P}}%
\global\long\def\T{\mathcal{T}}%
\global\long\def\EE{\mathcal{E}}%
\global\long\def\d{\mathop{}\!\mathrm{d}}%
\global\long\def\defeq{\overset{\mathrm{def}}{=}}%

\title{The Randomized Midpoint Method for Log-Concave Sampling}
\author{Ruoqi Shen\\
University of Washington\\
\texttt{shenr3@cs.washington.edu}\and Yin Tat Lee\\
University of Washington and Microsoft Research\\
\texttt{yintat@uw.edu}}

\maketitle
\selectlanguage{english}%
\global\long\def\defeq{\overset{\mathrm{def}}{=}}%

\begin{abstract}
Sampling from log-concave distributions is a well researched problem
that has many applications in statistics and machine learning. We
study the distributions of the form $p^{*}\propto\exp(-f(x))$, where
$f:\mathbb{R}^{d}\rightarrow\mathbb{R}$ has an $L$-Lipschitz gradient
and is $m$-strongly convex. In our paper, we propose a Markov chain
Monte Carlo (MCMC) algorithm based on the underdamped Langevin diffusion
(ULD). It can achieve $\epsilon\cdot D$ error (in 2-Wasserstein distance)
in $\tilde{O}\left(\kappa^{7/6}/\epsilon^{1/3}+\kappa/\epsilon^{2/3}\right)$
steps, where $D\defeq\sqrt{\frac{d}{m}}$ is the effective diameter
of the problem and $\kappa\defeq\frac{L}{m}$ is the condition number.
Our algorithm performs significantly faster than the previously best
known algorithm for solving this problem, which requires $\tilde{O}\left(\kappa^{1.5}/\epsilon\right)$
steps \cite{chen2019optimal,dalalyan2018sampling}. Moreover, our
algorithm can be easily parallelized to require only $O(\kappa\log\frac{1}{\epsilon})$
parallel steps. 

To solve the sampling problem, we propose a new framework to discretize
stochastic differential equations. We apply this framework to discretize
and simulate ULD, which converges to the target distribution $p^{*}$.
The framework can be used to solve not only the log-concave sampling
problem, but any problem that involves simulating (stochastic) differential
equations.\selectlanguage{american}%
\end{abstract}

\selectlanguage{english}%

\section{Introduction}

In this paper, we study the problem of sampling from a high-dimensional
log-concave distribution. This problem is central in statistics, machine
learning and theoretical computer science, with applications such
as Bayesian estimation \cite{andrieu2003introduction}, volume computation
\cite{vempala2010recent} and bandit optimization \cite{russo2018tutorial}.
In a seminal 1989 result, Dyer, Frieze and Kannan \cite{dyer1991random}
first presented a polynomial-time algorithm (for an equivalent problem)
that takes $\tilde{O}(d^{23}\log\frac{1}{\epsilon})$ steps on any
$d$ dimensional log-concave distribution to achieve target accuracy
$\epsilon$. After three decades of research in Markov chain Monte
Carlo (MCMC) and convex geometry \cite{lovasz1990mixing,applegate1991sampling,dyer1991computing,lovasz1993random,Kannan1997,lovasz2006simulated,Cousins:2015:BKG:2746539.2746563,Lee:2017:GWP:3055399.3055416,Lee:2018:CRR:3188745.3188774,mangoubi2019faster},
results have been improved to $\tilde{O}(d^{4}\log\frac{1}{\epsilon})$
steps for general log-concave distributions and slightly better for
distributions given in a certain form. Unfortunately, $d\log\frac{1}{\epsilon}$
steps are necessary even for a special case of log-concave sampling,
i.e., convex optimization \cite{alma991000409079704987}. To avoid
this lower bound, there has been a recent surge of interest in obtaining
a faster algorithm via assuming some properties on the distribution.

We call a distribution \emph{log-concave} if its density is proportional
to $e^{-f(x)}$ with a convex function $f$. For the standard assumption
that $f$ is $m$-strongly convex with an $L$-Lipschitz gradient
(see Section $\text{\ref{subsec:f}}$), the current best algorithms
have at least a linear $d$ or $1/\epsilon$ dependence or a large
dependence on the condition number $\kappa\defeq\frac{L}{m}$. In
this paper, we present an algorithm with no dependence on $d$ and
a much smaller dependence on $\kappa$ and $\epsilon$ than shown
in previous research. Moreover, our algorithm is the first algorithm
with better than $1/\epsilon$ dependence that is not Metropolis-adjusted
and does not make any extra assumption, such as high-order smoothness
\cite{mangoubi2017rapid,mangoubi2018dimensionally,chatterji2018theory,mou2019high}.

To explain our main result, we note that this problem has an effective
diameter $D\defeq\sqrt{\frac{d}{m}}$ because the distance between
the minimizer $x^{*}$ of $f$ and a random point $y\sim e^{-f}$
satisfies $\E_{y\sim e^{-f}}\|x^{*}-y\|^{2}\leq\frac{d}{m}$\cite{durmus2016high}.
Therefore, a natural problem definition\footnote{Previous papers addressing this problem defined $\epsilon$ as $W_{2}(x,e^{-f})\leq\epsilon$.
This definition is not scale invariant, i.e., the number of steps
changes when we scale $f$. In comparison, our definition yields results
that are invariant under: (1) the scaling of $f$, namely, replacing
$f(x)$ by $\alpha f(x)$ for $\alpha>0$, and (2) the tensor power
of $f$, namely, replacing $f(x)$ by $g(x)\defeq\sum_{i}f(x_{i})$.
Our new definition of $\epsilon$ also clarifies definitions in previous
research. Under the prior definition of $\epsilon$, the algorithms
\cite{durmus2016high,cheng2017underdamped,chen2019optimal} take $\tilde{O}(\kappa^{2}(\sqrt{\frac{d}{m}}/\epsilon)^{2})$,
$\tilde{O}(\kappa^{2}\sqrt{\frac{d}{m}}/\epsilon)$, and $\tilde{O}(\kappa^{1.5}\sqrt{\frac{d}{m}}/\epsilon)$
steps, respectively. Our new definition shows that these different
dependences on $d$ and $m$ all relate to their dependence on $\epsilon$.} is to find a random $x$ that makes the Wasserstein distance small:
\begin{equation}
W_{2}(x,y)\leq\epsilon\cdot D.\label{eq:W2_distance}
\end{equation}
This choice of distance is also common in previous papers \cite{durmus2016high,durmus2017nonasymptotic,cheng2017underdamped,mangoubi2017rapid,lee2018algorithmic,mangoubi2018dimensionally,chatterji2018theory}.

For $\epsilon=1$, we can simply output the minimizer $x^{*}$ of
$f$ as the ``random'' point. We first consider the question how
quickly we can find a random point satisfying $\epsilon=\frac{1}{2}$.
For convex optimization under the same assumption, it takes $\sqrt{\kappa}$
iterations via acceleration methods or $d$ iterations via cutting
plane methods, and these results are tight. For sampling, the current
fastest algorithms take either $\tilde{O}(\kappa^{1.5})$ steps \cite{chen2019optimal,dalalyan2018sampling}
or $\tilde{O}(d^{4})$ steps \cite{lovasz2006hit}. Although there
is no rigorous lower bound for this problem, it is believed that $\min(\text{\ensuremath{\kappa}},d^{2})$
is the natural barrier.\footnote{The corresponding optimization problem takes at least $\min(\sqrt{\kappa},d)$
steps \cite{alma991000409079704987}. If we represent each point the
optimization algorithm visited by a vertex and each step the algorithm
takes by an edge, then the existing lower bound in fact shows that
this graph has a diameter of at least $\min(\sqrt{\kappa},d)$. Since
a random walk on a graph of diameter $D$ takes $D^{2}$ to mix, a
random walk on the graph takes at least $\min(\sqrt{\kappa},d)^{2}$
steps.} This paper presents an algorithm that takes only $\tilde{O}(\kappa^{7/6})$
steps, much closer to the natural barrier of $\kappa$ for the high-dimensional
regime.

For general $0<\epsilon<1$, our algorithm takes $\tilde{O}(\kappa^{7/6}/\epsilon^{1/3}+\kappa/\epsilon^{2/3})$
steps, which is almost linear in $\kappa$ and sub-linear in $\epsilon$
. It has significantly better dependence on both $\kappa$ and $\epsilon$
than previous algorithms. (See the detailed comparison in Table \ref{tb: Runtime}.)
Moreover, if we query gradient $\nabla f$ at multiple points in parallel
in each step, we can improve the number to $O(\kappa\log\frac{1}{\epsilon})$
steps. 
\begin{table}[h]
\centering{}%
\begin{tabular}{|>{\centering}m{4.5cm}|c|c|}
\hline 
\multirow{1}{4.5cm}{\foreignlanguage{american}{}} & \multicolumn{2}{c|}{\textbf{$\mathbf{\#}$ Step}}\tabularnewline
\cline{2-3} 
\textbf{Algorithm} & \textbf{Warm Start} & \textbf{Cold Start}\tabularnewline
\hline 
\hline 
Hit-and-Run\cite{lovasz2006hit} & $\tilde{O}\left(d^{3}\log(\frac{1}{\epsilon})\right)$ & $\tilde{O}\left(d^{4}\log(\frac{1}{\epsilon})\right)$\tabularnewline
\hline 
Langevin Diffusion\cite{durmus2016high,doi:10.1111/rssb.12183} & \multicolumn{2}{c|}{$\tilde{O}\left(\kappa^{2}/\epsilon^{2}\right)$}\tabularnewline
\hline 
Underdamped Langevin Diffusion \cite{cheng2017underdamped} & \multicolumn{2}{c|}{$\tilde{O}\left(\kappa^{2}/\epsilon\right)$}\tabularnewline
\hline 
Underdamped Langevin Diffusion2

\cite{dalalyan2018sampling} & \multicolumn{2}{c|}{$\tilde{O}\left(\kappa^{1.5}/\epsilon+\kappa^{2}\right)$}\tabularnewline
\hline 
High-Order Langevin Diffusion\cite{mou2019high} & \multicolumn{2}{c|}{$\tilde{O}\left(\kappa^{19/4}/\epsilon^{1/2}+\kappa^{13/3}/\epsilon^{2/3}\right)$}\tabularnewline
\hline 
Metropolis-Adjusted Langevin Algorithm\cite{dwivedi2018log} & $\tilde{O}\left(\left(\kappa d+\kappa^{1.5}\sqrt{d}\right)\log(\frac{1}{\epsilon})\right)$ & $\tilde{O}\left(\left(\kappa d^{2}+\kappa^{1.5}d^{1.5}\right)\log(\frac{1}{\epsilon})\right)$\tabularnewline
\hline 
Hamiltonian Monte Carlo with Euler Method \cite{mangoubi2017rapid} & \multicolumn{2}{c|}{$\tilde{O}\left(\kappa^{6.5}/\epsilon\right)$}\tabularnewline
\hline 
Hamiltonian Monte Carlo with Collocation Method \cite{lee2018algorithmic} & \multicolumn{2}{c|}{$\tilde{O}\left(\kappa^{1.75}/\epsilon\right)$}\tabularnewline
\hline 
Hamiltonian Monte Carlo with Collocation Method 2 \cite{chen2019optimal} & \multicolumn{2}{c|}{$\tilde{O}\left(\kappa^{1.5}/\epsilon\right)$}\tabularnewline
\hline 
Underdamped Langevin Diffusion with Randomized Midpoint Method (This
Paper) & \multicolumn{2}{c|}{$\tilde{O}\left(\kappa^{7/6}/\epsilon^{1/3}+\kappa/\epsilon^{2/3}\right)$}\tabularnewline
\hline 
\end{tabular}\caption{Summary of iteration complexity. Except for Hit-and-Run, each step
involves $O(1)$-gradient computation. Hit-and-Run takes $\tilde{O}(1)$
function value computations in each step.\label{tb: Runtime}}
\end{table}

\subsection{Contributions\label{subsec:Contributions}}

We propose a new framework to discretize stochastic differential equations
(SDEs), which is a crucial step of log-sampling algorithms. Since
our techniques can also be applied to ordinary differential equations
(ODEs), we focus on the following ODE here:
\[
\frac{\d x}{\d t}=F(x(t)).
\]
There are two main frameworks to discretize a differential equation.
One is the Taylor expansion, which approximates $x(t)$ by $x(0)+x'(0)t+x''(0)\frac{t^{2}}{2}+\cdots$.
Our paper uses the second framework, called the \emph{collocation
method}. This method uses the fact that the differential equation
is equivalent to the integral equation $x=\mathcal{T}(x)$, where
$\mathcal{T}$ maps continuous functions to continuous functions:
\[
\mathcal{T}(x)(t)=x(0)+\int_{0}^{t}F(x(s))\d s\text{ for all }t\geq0.
\]
Since $x$ is a fixed point of $\mathcal{T}$, we can approximate
$x$ by computing $\mathcal{T}(\mathcal{T}(\cdots(\mathcal{T}(x_{0}))\cdots))$
for some approximate initial function $x_{0}$. Algorithmically, two
key questions are how to: (1) show when and how quickly $\mathcal{T}$
iterations converge, and (2) compute the integration. The convergence
rate of $\mathcal{T}$ was shown by the Picard--Lindelöf Theorem
in the 1890s \cite{lindelof1894application,picard1898methodes} and
was key to achieving $O(\kappa^{1.75})$ and $O\left(\kappa^{1.5}\right)$
in the previous papers \cite{lee2018algorithmic,chen2019optimal}.
To approximate the integration, one standard approach is to approximate
\[
\int_{0}^{t}F(x(s))\d s\sim\sum_{i}w_{i}F(x(s_{i}))
\]
for some carefully chosen $w_{i}$ and $s_{i}$. The key drawback
of this approach is its introduction of a deterministic error, which
accumulates linearly to the number of steps. Since we expect to take
at least $\kappa$-many iterations, the approximation error must be
$\kappa$ times smaller than the target accuracy.

In this paper, we improve upon the collocation method for sampling
by developing a new algorithm, called the \emph{randomized midpoint
method}, that yields three distinct benefits:
\begin{enumerate}
\item We generalize fixed point iteration to stochastic differential equations
and hence avoid the cost of reducing SDEs to ODEs, as was done in
\cite{lee2018algorithmic}.
\item We greatly reduce the error accumulation by simply approximating $\int_{0}^{t}F(x(s))ds$
by $t\cdot F(x(s))$ where $s$ is randomly chosen from $0$ to $t$
uniformly.
\item We show that two iterations of $\mathcal{T}$ suffice to achieve the
best theoretical guarantee.
\end{enumerate}
Although we discuss only strongly convex functions with a Lipschitz
gradient, we believe our framework can be applied to other classes
of functions, as well. By designing suitable unbiased estimators of
integrals, researchers can easily use our approach to obtain faster
algorithms for solving SDEs that are unrelated to sampling problems. 

\subsection{Paper Organization}

Section $\text{\ref{sec:Background}}$ provides background information
on solving the log-concave sampling problem, while Section $\text{\ref{sec:Notation-and-Definition}}$
introduces our notations and assumptions about the function $f$.
We introduce our algorithm in Section $\text{\ref{sec:Algorithm-and-Results}}$,
where we present the main result of our paper. We show our proofs
in appendices: Appendix $\text{\ref{sec:Brownian-Motion-Simulation}}$--how
we simulate the Brownian motion; Appendix $\text{\ref{sec:Properties-of-the}}$--important
properties of ULD and the Brownian motion; Appendix $\text{\ref{sec:Proof-of-Lemma}}$--
bounds for the discretization error of our algorithm; Appendix $\text{\ref{sec:Bounds-on-}}$--a
bound on the average value of $\left\Vert \nabla f(x_{n})\right\Vert $
and $\left\Vert v_{n}\right\Vert $ in our algorithm, which is useful
for bounding the discretization error; Appendix $\text{\ref{sec:Proof-of-Theorem}}$--proofs
for the main result of our paper; Appendix $\text{\ref{sec:Discretization-Error-of}}$--additional
proofs on how to parallelize our algorithm.\selectlanguage{american}%

\selectlanguage{english}%

\section{Background\label{sec:Background}}

Many different algorithms have been proposed to solve the log-concave
sampling problem. The general approach uses a MCMC-based algorithm
that often includes two steps. The first step involves the choice
of a Markov process with a stationary distribution equal or close
to the target distribution. The second step is discretizing the process
and simulating it until the distribution of the points generated is
sufficiently close to the target distribution.

\subsection{Choosing the Markov Process}

One commonly used Markov process is the Langevin diffusion (LD) \cite{roberts1996exponential,203584,durmus2019analysis}.
LD evolves according to the SDE
\begin{eqnarray}
\d x(t) & = & -\nabla f(x(t))+\sqrt{2}\d B_{t},\label{eq:SDE}
\end{eqnarray}
where $B_{t}$ is the standard Brownian motion. Under the assumption
that $f$ is $L$-smooth and $m$-strongly convex (see Section $\text{\ref{subsec:f}}$)
with $\kappa=\frac{L}{m}$ as the condition number, \cite{durmus2016high,doi:10.1111/rssb.12183,cheng2017convergence}
show that algorithms based on LD can achieve less than $\epsilon$
error in $\tilde{O}\left(\frac{\kappa^{2}}{\epsilon^{2}}\right)$
steps. Other related works include LD with stochastic gradient \cite{dalalyan2019user,zhang2017hitting,raginsky2017non,chatterji2018theory}
and LD in the non-convex setting \cite{raginsky2017non,cheng2018sharp}.

One important breakthrough introduced the Hamiltonian Monte Carlo
(HMC), originally proposed in \cite{kramers1940brownian}. In this
process, SDE (\ref{eq:SDE}) is approximated by a piece-wise curve,
where each piece is governed by an ODE called the Hamiltonian dynamics.
The Hamiltonian dynamics maintains a velocity $v$ in addition to
a position $x$ and conserves the value of the Hamiltonian $H(x,v)=f(x)+\frac{1}{2}\left\Vert v\right\Vert ^{2}.$
HMC has been widely studied in \cite{neal2011mcmc,ma2015complete,mangoubi2017rapid,mangoubi2018dimensionally,lee2018algorithmic,chen2019optimal,Lee:2018:CRR:3188745.3188774}.
The works \cite{chen2019optimal,dalalyan2018sampling} show that algorithms
based on HMC can achieve less than $\epsilon$ error in $\tilde{O}\left(\frac{\kappa^{1.5}}{\epsilon}\right)$
steps.

The underdamped Langevin diffusion (ULD) can be viewed as a version
of HMC that replaces multiple ODEs with one SDE; it has been studied
in \cite{cheng2017underdamped,eberle2017couplings,dalalyan2018sampling}.
ULD follows the SDE:
\begin{align}
\d v(t)=-2v(t)\d t-u\nabla f(x(t))\d t+2\sqrt{u}\d B_{t}, & \qquad\d x(t)=v(t)\d t,\label{eq:ULD}
\end{align}
where $u=\frac{1}{L}$. \cite{cheng2017underdamped} shows that even
a basic discretization of ULD has a fast convergence rate that can
achieve less than $\epsilon$ error in $\tilde{O}\left(\frac{\kappa^{2}}{\epsilon}\right)$
steps. Recently, it was shown that ULD can be viewed as an accelerated
gradient descent for sampling \cite{arxiv190200996}. This suggests
that ULD might be one of the right dynamic for sampling in the same
way as the accelerated gradient descent method is appropriate for
convex optimization. For this reason, our paper focuses on how to
discretize ULD. We note that our framework can be applied to both
LD and HMC to improve on previous results for these dynamics as well. 

\subsection{Discretizing the Process}

To simulate the random process mentioned, previous works usually apply
the Euler method \cite{cheng2017underdamped,durmus2016high} or the
Leapfrog method \cite{mangoubi2017rapid,mangoubi2018dimensionally}
to discretize the SDEs or the ODEs. In Section \ref{subsec:Random-Midpoint-Method},
we introduce a 2-step fixed point iteration method to solve general
differential equations. We apply this method to ULD and significantly
reduce the discretization error compared to existing methods. In particular,
ULD can achieve less than $\epsilon$ error in $\tilde{O}\left(\frac{\kappa^{7/6}}{\epsilon^{1/3}}+\frac{\kappa}{\epsilon^{2/3}}\right)$
steps. Table $\text{\ref{tb: Runtime}}$ summarizes the number of
steps needed by previous algorithms versus our algorithm. Moreover,
with slightly more effort, our algorithm can be parallelized so that
it needs only $O\left(\kappa\log\frac{1}{\epsilon}\right)$ parallel
steps.

On top of the discretization method, one can use a Metropolis-Hastings
accept-reject step to ensure that the post-discretization random process
results in a stationary distribution equal to the target distribution
\cite{belisle1993hit,lovasz1993random,roberts1996geometric,mengersen1996rates,lovasz1999hit,lovasz2006hit,lovasz2007geometry}.
\cite{lovasz2006hit} gives the current best algorithm for arbitrary
log-concave distribution. Originally proposed in \cite{roberts1996exponential,roberts1996geometric},
the Metropolis Adjusted Langevin Algorithm (MALA) \cite{Roberts97optimalscaling,jarner2000geometric,pillai2012optimal,10.1093/imanum/drs003,xifara2014langevin,Pereyra2016}
applies the Metropolis-Hastings accept-reject step to the Langevin
diffusion. \cite{dwivedi2018log} shows MALA can achieve $\epsilon$
error in total variation distance in $\tilde{O}\left(\left(\kappa d+\kappa^{1.5}\sqrt{d}\right)\log\left(\frac{\beta}{\epsilon}\right)\right)$
steps for $\beta$-warm start. Unlike other algorithms that have a
$\frac{1}{\epsilon^{O(1)}}$ dependence on $\epsilon$, MALA depends
logarithmically on $\epsilon$. However, $\beta$ usually depends
exponentially on the dimension $d$, which results in a $\Omega(d^{1.5})$
dependence in total. Since this paper focuses on achieving a dimension
independent result, we do not discuss how to combine our process with
a Metropolis-Hastings step in this paper.

Finally, we note that all results--including ours--can be improved
if we assume that $f$ has bounded higher-order derivatives. To ensure
a fair comparison in Table $\text{\ref{tb: Runtime}}$, we only include
results that only assume $f$ is strongly convex and has a Lipschitz
gradient.\selectlanguage{american}%

\section{Notations and Definitions\label{sec:Notation-and-Definition}}

For any function $f$, we use $\tilde{O}(f)$ to denote the class
$O\left(f\right)\cdot\log^{O(1)}(f)$. For vector $v\in\R^{d}$, we
use $\left\Vert v\right\Vert $ to denote the Euclidean norm of $v$.

\subsection{Assumptions on $f$}

\label{subsec:f}

We assume that the function $f$ is a twice continuously differentiable
function from $\R^{d}$ to $\R$ that has an $L$-Lipschitz continuous
gradient and is $m$-strongly convex. That is, there exist positive
constants $L$ and $m$ such that for all $x,y\in\R^{d}$, 
\begin{align*}
\left\Vert \nabla f(x)-\nabla f(y)\right\Vert \leq L\left\Vert x-y\right\Vert , & \text{ and }f(y)\geq f(x)+\left\langle \nabla f(x),y-x\right\rangle +\frac{m}{2}\left\Vert x-y\right\Vert ^{2}.
\end{align*}
It is easy to show that these inequalities are equivalent to $mI_{d}\preceq\nabla^{2}f(x)\preceq LI_{d},$
where $I_{d}$ is the identity matrix of dimension $d$. Let $\kappa=\frac{L}{m}$
be the condition number. We assume that we have access to an oracle
that, given a point $x\in\R^{d}$, can return the gradient of $f$
at point $x$, $\nabla f(x)$.

\subsection{Wasserstein Distance}

The $p$th Wasserstein distance between two probability measures $\mu$
and $\nu$ is defined as
\begin{eqnarray*}
W_{p}\left(\mu,\nu\right) & = & \left(\inf_{\left(X,Y\right)\in\mathcal{C}\left(\mu,\nu\right)}\E\left[\left\Vert X-Y\right\Vert ^{p}\right]\right)^{1/p},
\end{eqnarray*}
where $\mathcal{C}\left(\mu,\nu\right)$ is the set of all couplings
of $\mu$ and $\nu$. In this paper, for any $0<\epsilon<1$, we study
the number of steps needed so that the $W_{2}$ distance between the
distribution of the point our algorithms generate and the target distribution
is smaller than $\epsilon\cdot D$.

\section{Algorithms and Results\label{sec:Algorithm-and-Results}}

\subsection{Underdamped Langevin Diffusion (ULD)}

ULD is a random process that evolves according to $\text{\ensuremath{\left(\ref{eq:ULD}\right)}}$.
Our paper studies $\text{\ensuremath{\left(\ref{eq:ULD}\right)}}$
with $u=\frac{1}{L}$. Under mild conditions, it can be shown that
the stationary distribution of $\text{\ensuremath{\left(\ref{eq:ULD}\right)}}$
is proportional to $\exp\left(-f(x)+L\left\Vert v\right\Vert ^{2}/2\right).$
Then, the marginal distribution of $x$ is proportional to $\exp\left(-f(x)\right).$
It can also be shown that the solution to $\text{\ensuremath{\left(\ref{eq:ULD}\right)}}$
has a contraction property \cite{cheng2017underdamped,eberle2017couplings},
shown in the following lemma.
\begin{lem}[Theorem 5 of \cite{cheng2017underdamped}]
\label{lem:ULD} Let $\left(x_{0},v_{0}\right)$ and $\left(y_{0},w_{0}\right)$
be two arbitrary points in $\R^{d}\times\R^{d}.$ Let $\left(x_{t},v_{t}\right)$
and $\left(y_{t},w_{t}\right)$ be the exact solutions of the underdamped
Langevin diffusion after time $t$. If $\left(x_{t},v_{t}\right)$
and $\left(y_{t},w_{t}\right)$ are coupled through a shared Brownian
motion, then, 
\[
\E\left[\left\Vert x_{t}-y_{t}\right\Vert ^{2}+\left\Vert \left(x_{t}+v_{t}\right)-\left(y_{t}+w_{t}\right)\right\Vert ^{2}\right]\leq e^{-\frac{t}{\kappa}}\E\left[\left\Vert x_{0}-y_{0}\right\Vert ^{2}+\left\Vert \left(x_{0}+v_{0}\right)-\left(y_{0}+w_{0}\right)\right\Vert ^{2}\right].
\]
\end{lem}

This contraction bound can be very useful for showing the convergence
of the continuous process $\text{\ensuremath{\left(\ref{eq:ULD}\right)}}$.
In our algorithm, we discretize the continuous process to implement
it; therefore we need to use this contraction bound together with
a discretization error bound to show the guarantee of our algorithm.
In Section $\text{\ref{subsec:Random-Midpoint-Method}}$, we show
how we discretize $\left(\ref{eq:ULD}\right)$.

\subsection{Randomized Midpoint Method \label{subsec:Random-Midpoint-Method}}

\selectlanguage{french}%
\begin{algorithm}
\caption{\label{al:lan}\foreignlanguage{american}{Randomized Midpoint Method
for ULD}}

\begin{algorithmic}[1]

\selectlanguage{american}%
\STATE${\bf Procedure}$ RandomMidpoint$(x_{0},v_{0},N,h)$

\selectlanguage{french}%
\STATE ${\bf For}\;n=0,...,N-1$

\STATE$\quad$Randomly sample $\alpha$ uniformly from $[0,1].$

\STATE$\quad$Generate Gaussian random variable $\left(W_{1}^{(n)},W_{2}^{(n)},W_{3}^{(n)}\right)\in\R^{3d}$
as in Appendix $\text{\ref{sec:Brownian-Motion-Simulation}}$

\STATE$\quad x_{n+\frac{1}{2}}=x_{n}+\frac{1}{2}\left(1-e^{-2\alpha h}\right)v_{n}-\frac{1}{2}u\left(\alpha h-\frac{1}{2}(1-e^{-2\alpha h})\right)\nabla f(x_{n})+\sqrt{u}W_{1}^{(n)}.$

\STATE$\quad$$x_{n+1}=x_{n}+\frac{1}{2}\left(1-e^{-2h}\right)v_{n}-\frac{1}{2}uh\left(1-e^{-2(h-\alpha h)}\right)\nabla f(x_{n+\frac{1}{2}})+\sqrt{u}W_{2}^{(n)}.$\\
\STATE$\quad$$v_{n+1}=v_{n}e^{-2h}-uhe^{-2(h-\alpha h)}\nabla f(x_{n+\frac{1}{2}})+2\sqrt{u}W_{3}^{(n)}.$

\STATE ${\bf end}{\bf \;for}$

\STATE${\bf end}{\bf \;procedure}$

\end{algorithmic}
\end{algorithm}
\selectlanguage{american}%
Our step size for each iteration is $h$. In iteration $n$ of our
algorithm, to simulate (\ref{eq:ULD}), we need to approximate the
solution to SDE (\ref{eq:ULD}) at time $h$, $\left(x_{n}^{*}(h),v_{n}^{*}(h)\right)$,
with initial value, $\left(x_{n},v_{n}\right)$. The simplest way
to do so is to use the Euler method:
\begin{align*}
v_{n}(h) & =(1-2h)v_{n}-uh\nabla f(x_{n})+2\sqrt{uh}\zeta,\quad x_{n}(h)=x_{n}+hv_{n},
\end{align*}
where $\zeta\in\R^{d}$ is drawn from the standard normal distribution.
This discretization was considered in \cite{durmus2017nonasymptotic,doi:10.1111/rssb.12183}
due to its simplicity.

As discussed in Section \ref{subsec:Contributions}, we improve the
accuracy by studying the integral formulation of (\ref{eq:ULD}):
\begin{align}
x_{n}^{*}(t) & =x_{n}+\frac{1-e^{-2t}}{2}v_{n}-\frac{u}{2}\int_{0}^{t}\left(1-e^{-2(t-s)}\right)\nabla f(x_{n}^{*}(s))\d s+\sqrt{u}\int_{0}^{t}\left(1-e^{-2(t-s)}\right)\d B_{s},\nonumber \\
v_{n}^{*}(t) & =v_{n}e^{-2t}-u\left(\int_{0}^{t}e^{-2(t-s)}\nabla f(x_{n}^{*}(s))\d s\right)+2\sqrt{u}\int_{0}^{t}e^{-2(t-s)}\d B_{s}.\label{eq:ULD_sol}
\end{align}
\cite{cheng2017underdamped} considered the same integral formulation
and used $\nabla f(x_{n})$ to approximate $\nabla f(x_{n}^{*}(t))$
for $t\in[0,h]$ to get the following algorithm:
\begin{align*}
\hat{x}_{n}(h) & =x_{n}+\frac{1-e^{-2h}}{2}v_{n}-\frac{u}{2}\int_{0}^{h}\left(1-e^{-2(h-s)}\right)\nabla f(x_{n})\d s+\sqrt{u}\int_{0}^{h}\left(1-e^{-2(h-s)}\right)\d B_{s},\\
\hat{v}_{n}(h) & =v_{n}e^{-2h}-u\left(\int_{0}^{h}e^{-2(h-s)}\nabla f(x_{n})\d s\right)+2\sqrt{u}\int_{0}^{h}e^{-2(h-s)}\d B_{s}.
\end{align*}
However, this approximation method can still generate a relatively
large error. Our paper proposes a new method, the randomized midpoint
method, to solve (\ref{eq:ULD_sol}), which yields a more accurate
approximation and significantly reduces the total runtime of the algorithm.

We first need to identify an accurate estimator of the integral $\int_{0}^{h}\left(1-e^{-2(h-s)}\right)\nabla f(x_{n}^{*}(s))\d s.$
To do so, we sample a random number $\alpha$ uniformly from $[0,1]$
so that $\alpha h$ gives a random point from $[0,h].$ Then, $h\left(1-e^{-2(h-\alpha h)}\right)\nabla f(x_{n}^{*}(\alpha h))$
is an accurate estimator of the integral \\
$\int_{0}^{h}\left(1-e^{-2(h-s)}\right)\nabla f(x_{n}^{*}(s))\d s.$
We can further show that this estimator is unbiased.

For brevity, we use $x_{n+\frac{1}{2}}$ to denote our approximation
of $x_{n}^{*}(\alpha h)$. To approximate $x_{n}^{*}(\alpha h)$,
we use equation (\ref{eq:ULD_sol}) again:
\[
x_{n+\frac{1}{2}}=x_{n}+\frac{1-e^{-2\alpha h}}{2}v_{n}-\frac{u}{2}\int_{0}^{\alpha h}\left(1-e^{-2(\alpha h-s)}\right)\nabla f(x_{n})\!\mathrm{d}s+\sqrt{u}\int_{0}^{\alpha h}\left(1-e^{-2(\alpha h-s)}\right)\!\mathrm{d}B_{s}.
\]
Then, $\left(x_{n}^{*}(h),v_{n}^{*}(h)\right)$ can be approximated
as
\begin{align*}
x_{n+1} & =x_{n}+\frac{1-e^{-2h}}{2}v_{n}-\frac{u}{2}h\left(1-e^{-2(h-\alpha h)}\right)\nabla f(x_{n+\frac{1}{2}})+\sqrt{u}\int_{0}^{h}\left(1-e^{-2(h-s)}\right)\d B_{s},\\
v_{n+1} & =v_{n}e^{-2h}-uhe^{-2(h-\alpha h)}\nabla f(x_{n+\frac{1}{2}})+2\sqrt{u}\int_{0}^{h}e^{-2(h-s)}\d B_{s}.
\end{align*}

Note that we can view (\ref{eq:ULD_sol}) as the fixed point of the
operator $\T$, $x_{n}^{*}=\mathcal{T}(x_{n}^{*})$, where for all
$t$,
\begin{align}
\mathcal{T}(x)(t) & =x_{n}+\frac{1-e^{-2t}}{2}v_{n}-\frac{u}{2}\int_{0}^{t}\left(1-e^{-2(t-s)}\right)\nabla f(x(s))\d s+\sqrt{u}\int_{0}^{t}\left(1-e^{-2(t-s)}\right)\d B_{s}.\label{eq:op_t}
\end{align}
Then, our randomized algorithm is essentially approximating $\mathcal{T}(\mathcal{T}(x_{n}))$.
Under the assumption $f$ is twice differentiable, we show that two
iterations suffice to achieve the best theoretical guarantee, but
we suspect more iterations might be useful if $f$ has higher order
derivatives. As emphasized in Section \ref{subsec:Contributions},
the way we obtain our algorithm forms a general framework that can
be applied to other SDEs.

In Lemma $\text{\ref{lem:mean_var}}$, we show that the stochastic
terms $W_{1}=\int_{0}^{\alpha h}\left(1-e^{-2(\alpha h-s)}\right)\d B_{s}$,\\
$W_{2}=\int_{0}^{h}\left(1-e^{-2(h-s)}\right)\d B_{s},$ and $W_{3}=\int_{0}^{h}e^{-2(h-s)}\d B_{s}$
conditional on the choice of $\alpha$ follow a multi-dimensional
Gaussian distribution and therefore can be easily sampled. The steps
mentioned above are summarized in Algorithm $\text{\ref{al:lan}}$.
Using this randomized midpoint method, we can solve (\ref{eq:ULD_sol})
much more accurately than previous works. We show that the discretization
error satisfies:

\begin{lem}
\label{lem:error}For each iteration $n$ of Algorithm $\text{\ref{al:lan}}$,
let $\E_{\alpha}$ be the expectation taken over the random choice
of $\alpha$ in iteration $n$. Let $\E$ be the expectation taken
over other randomness in iteration $n$. Let $\left(x_{n}^{*}(t),v_{n}^{*}(t)\right)_{t\in[0,h]}$
be the solution of the exact underdamped Langevin diffusion starting
from $\left(x_{n},v_{n}\right)$ coupled through a shared Brownian
motion with $x_{n+\frac{1}{2}},$ $v_{n}$ and $x_{n+1}.$ Assume
that $h\leq\frac{1}{20}$ and $u=\frac{1}{L}$. Then, $x_{n+1}$ and
$v_{n+1}$ of Algorithm $\text{\ref{al:lan}}$ satisfy
\begin{eqnarray*}
\E\left\Vert \E_{\alpha}x_{n+1}-x_{n}^{*}(h)\right\Vert ^{2} & \leq & O\left(h^{10}\left\Vert v_{n}\right\Vert ^{2}+u^{2}h^{12}\left\Vert \nabla f(x_{n})\right\Vert ^{2}+udh^{11}\right),\\
\E\left\Vert x_{n+1}-x_{n}^{*}(h)\right\Vert ^{2} & \leq & O\left(h^{6}\left\Vert v_{n}\right\Vert ^{2}+u^{2}h^{4}\left\Vert \nabla f(x_{n})\right\Vert ^{2}+udh^{7}\right),\\
\E\left\Vert \E_{\alpha}v_{n+1}-v_{n}^{*}(h)\right\Vert ^{2} & \leq & O\left(h^{8}\left\Vert v_{n}\right\Vert ^{2}+u^{2}h^{10}\left\Vert \nabla f(x_{n})\right\Vert ^{2}+udh^{9}\right),\\
\E\left\Vert v_{n+1}-v_{n}^{*}(h)\right\Vert ^{2} & \leq & O\left(h^{4}\left\Vert v_{n}\right\Vert ^{2}+u^{2}h^{4}\left\Vert \nabla f(x_{n})\right\Vert ^{2}+udh^{5}\right).
\end{eqnarray*}
\end{lem}

In Appendix $\text{\ref{sec:Bounds-on-}}$, we show that the average
value of $\left\Vert v_{n}\right\Vert ^{2}$ is of order $\tilde{O}\left(\frac{d}{L}\right)$;
that of $\left\Vert \nabla f(x_{n})\right\Vert ^{2}$ is of order
$\tilde{O}\left(Ld\right)$. Then, Lemma $\text{\ref{lem:error}}$
shows that the bias of the discretization is of order $\tilde{O}\left(h^{4}\sqrt{\frac{d}{L}}\right)$
and the standard deviation is of order $\tilde{O}\left(h^{2}\sqrt{\frac{d}{L}}\right)$,
which implies the error is larger when $h$ is larger. However, by
Lemma $\text{\ref{lem:ULD}}$, in order for the algorithm to converge
in a small number of steps, we need to avoid choosing an $h$ that
is too small. Therefore, it is important to choose the largest possible
$h$ that can still make the algorithm converge. By Lemma $\text{\ref{lem:ULD}}$,
it is sufficient to run our algorithm for $\tilde{O}\left(\frac{\kappa}{h}\right)$
iterations. Then, the bias will cumulate to $\tilde{O}\left(h^{4}\sqrt{\frac{d}{L}}\cdot\frac{\kappa}{h}\right)=\tilde{O}\left(h^{3}\sqrt{\frac{d\kappa}{m}}\right)$,
and the standard deviation will cumulate to $\tilde{O}\left(h^{2}\sqrt{\frac{d}{L}}\cdot\sqrt{\frac{\kappa}{h}}\right)=\tilde{O}\left(h^{1.5}\sqrt{\frac{d}{m}}\right)$.
Thus, in order to make the $W_{2}$ distance less than $\tilde{O}\left(\epsilon\sqrt{\frac{d}{m}}\right)$,
we show in Theorem $\text{\ref{thm:main}}$ that it is enough to choose
$h$ to be $\tilde{\Theta}\left(\min\left(\frac{\epsilon^{1/3}}{\kappa^{1/6}},\epsilon^{2/3}\right)\right)$.
This choice of $h$ yields the main result of our paper, which is
stated in Theorem $\text{\ref{thm:main}}$. (See Appendix $\text{\ref{sec:Proof-of-Theorem}}$
for the full proof.)
\begin{thm}[Main Result]
\label{thm:main}Let $f$ be a function such that $0\prec m\cdot I_{d}\preceq\nabla^{2}f(x)\preceq L\cdot I_{d}$
for all $x\in\R^{d}$. Let $Y$ be a random point drawn from the density
proportional to $e^{-f}.$ Let the starting point $x_{0}$ be the
point that minimizes $f(x)$ and $v_{0}=0$. For any $0<\epsilon<1,$
if we set the step size of Algorithm $\ref{al:lan}$ as $h=C\min\left(\frac{\epsilon^{1/3}}{\kappa^{1/6}}\log^{-1/6}\left(\frac{1}{\epsilon}\right),\epsilon^{2/3}\log^{-1/3}\left(\frac{1}{\epsilon}\right)\right)$,
for some small constant $C$ and run the algorithm for $N=\frac{2\kappa}{h}\log\left(\frac{20}{\epsilon^{2}}\right)\leq\tilde{O}\left(\frac{\kappa^{7/6}}{\epsilon^{1/3}}+\frac{\kappa}{\epsilon^{2/3}}\right)$
iterations, then Algorithm $\ref{al:lan}$ after $N$ iterations can
generate a random point $X$ such that $W_{2}(X,Y)\leq\epsilon\sqrt{\frac{d}{m}}.$
Furthermore, each iteration of Algorithm $\ref{al:lan}$ involves
computing $\nabla f$ exactly twice.
\end{thm}

\subsection{A More General Algorithm}

\selectlanguage{french}%
\begin{algorithm}
\caption{\label{al:lan2}\foreignlanguage{american}{Randomized Midpoint Method
for ULD (Parallel)}}

\begin{algorithmic}[1]

\selectlanguage{american}%
\STATE${\bf Procedure}$ RandomMidpoint\_P$(x_{0},v_{0},N,h,R)$

\selectlanguage{french}%
\STATE ${\bf For}\;n=0,...,N-1$

\STATE$\quad$Randomly sample \foreignlanguage{american}{$\alpha_{1}$,
..., $\alpha_{R}$ uniformly from $\left[0,\frac{1}{R}\right]$, $\left[\frac{1}{R},\frac{2}{R}\right]$,
..., $\left[\frac{R-1}{R},1\right]$.}

\STATE$\quad$Generate Gaussian r.v. $\left(W_{1,1}^{(n)},...,W_{1,R}^{(n)},W_{2}^{(n)},W_{3}^{(n)}\right)\in\R^{(R+2)d}$
similar to Appendix $\text{\ref{sec:Brownian-Motion-Simulation}}$

\STATE$\quad$$x_{n}^{(0,i)}=x_{n}$ for $i=1,...,R$.

\STATE$\quad$${\bf For}\;k=1,...,K-1,i=1,...,R$

\STATE$\quad$$\quad$$x_{n}^{(k,i)}=x_{n}+\frac{1}{2}\left(1-e^{-2\alpha_{i}h}\right)v_{n}$

\STATE$\quad$$\qquad$$-\frac{1}{2}u\sum_{j=1}^{i}\left[\int_{(j-1)\delta}^{\min(j\delta,\alpha_{i}h)}\left(1-e^{-2(\alpha_{i}h-s)}\right)\d s\cdot\nabla f(x_{n}^{(k-1,j)})\right]+\sqrt{u}W_{1,i}^{(n)}$

\STATE$\quad$ ${\bf end}{\bf \;for}$

\STATE$\quad$$x_{n+1}=x_{n}+\frac{1}{2}\left(1-e^{-2h}\right)v_{n}-\frac{1}{2}u\sum_{i=1}^{R}\delta\left(1-e^{-2(h-\alpha_{i}h)}\right)\nabla f(x_{n}^{(K-1,i)})+\sqrt{u}W_{2}^{(n)},$\\
\STATE$\quad v_{n+1}=v_{n}e^{-2h}-u\sum_{i=1}^{R}\delta e^{-2(h-\alpha_{i}h)}\nabla f(x_{n}^{(K-1,i)})+2\sqrt{u}W_{3}^{(n)}.$

\STATE ${\bf end}{\bf \;for}$

\STATE${\bf end}{\bf \;procedure}$

\end{algorithmic}
\end{algorithm}
\selectlanguage{american}%
Now we show how our algorithm can be parallelized. The algorithm studied
in this section can be viewed as a more general version of Algorithm
$\text{\ref{al:lan}}.$ Instead of choosing one random point from
$[0,h]$, we divide the time interval $[0,h]$ into $R$ pieces, each
of length $\delta=\frac{h}{R}$, and choose one random point from
each piece. That is, we randomly choose $\alpha_{1},$ $\alpha_{2}$,
..., $\alpha_{R}$ uniformly from $\left[0,\frac{1}{R}\right]$, $\left[\frac{1}{R},\frac{2}{R}\right]$,
..., $\left[\frac{R-1}{R},1\right]$. As in Algorithm $\text{\ref{al:lan}}$,
to approximate $\left(x_{n}^{*}(h),v_{n}^{*}(h)\right)$, we use
\begin{align*}
\tilde{x} & =x_{n}+\frac{1-e^{-2h}}{2}v_{n}-\frac{u}{2}\sum_{i=1}^{R}\delta\left(1-e^{-2(h-\alpha_{i}h)}\right)\nabla f(x_{n}^{*}(\alpha_{i}h))+\sqrt{u}\int_{0}^{h}\left(1-e^{-2(h-s)}\right)\mathrm{d}B_{s},\\
\tilde{v} & =v_{n}e^{-2h}-u\sum_{i=1}^{R}\delta e^{-2(h-\alpha_{i}h)}\nabla f(x^{*}(\alpha_{i}h))+2\sqrt{u}\int_{0}^{h}e^{-2(h-s)}\d B_{s},
\end{align*}

which gives an unbiased estimator of $\left(x_{n}^{*}(h),v_{n}^{*}(h)\right)$.
The next step is to approximate $x_{n}^{*}(\alpha_{i}h)$ for $i=1,..,R$.
We know that the solution $x_{n}^{*}$ is the fixed point of the operator
$\T$ defined in $\text{\ensuremath{\left(\ref{eq:op_t}\right)}}$.
To solve the fixed point of $\T$, we can use the fixed point iteration
method, which applies the operator $\T$ multiple times on some initial
point. By the Banach fixed point theorem, the resulting points can
converge to the fixed point of $\T$. Instead of applying $\T$, which
involves computing an integral, we apply the operator $\tilde{\T}$
, which approximates $\T$, on $X=\left(x^{(1)},...,x^{(R)}\right)$
,
\begin{align*}
\tilde{\T}\left(X\right)_{i} & =x_{n}+\frac{1}{2}\left(1-e^{-2\alpha_{i}h}\right)v_{n}-\frac{1}{2}u\sum_{j=1}^{i}\left[\int_{(j-1)\delta}^{\min(j\delta,\alpha_{i}h)}\left(1-e^{-2(\alpha_{i}h-s)}\right)\d s\cdot\nabla f(x^{(j)})\right]\\
 & +\sqrt{u}\int_{0}^{\alpha_{i}h}\left(1-e^{-2(\alpha_{i}h-s)}\right)\d B_{s}.
\end{align*}
We set the initial points to $x_{n}^{(0,j)}=x_{n}$ for $j=1,...,R$.
Then, we apply $\tilde{\T}$ for $K$ times and get $(x^{(K,1)},...,x^{(K,R)})=\tilde{\T}^{\circ K}(x^{(0,1)},...,x^{(0,R)})$.
The preceding steps are summarized in Algorithm $\text{\ref{al:lan2}}$.
It is easy to see Algorithm $\text{\ref{al:lan}}$ is a special case
of Algorithm $\text{\ref{al:lan2}}$ with $R=1$ and $K=2$. 

This algorithm can be parallelized since we can compute $\tilde{\T}(x^{(k,1)},...,x^{(k,R)})_{j}$
for each $j$ parallelly. It can be shown that it is sufficient to
choose $K$ to depend logarithmically on $\kappa$ and $\epsilon$.
Similar to Algorithm $\text{\ref{al:lan}}$, we can show that Algorithm
$\text{\ref{al:lan2}}$ has the guarantee that the bias of the discretization
is of order $\tilde{O}\left(\frac{h^{4}}{R}\sqrt{\frac{d}{L}}\right)$
and the standard deviation is of order $\tilde{O}\left(\frac{h^{2}}{R}\sqrt{\frac{d}{L}}\right)$
(Appendix $\text{\ref{sec:Discretization-Error-of}}$). Then, summing
from $\tilde{O}\left(\frac{\kappa}{h}\right)$ iterations, the total
bias would be $\tilde{O}\left(\frac{h^{4}}{R}\sqrt{\frac{d}{L}}\cdot\frac{\kappa}{h}\right)=\tilde{O}\left(\frac{h^{3}}{R}\sqrt{\frac{d\kappa}{m}}\right)$,
and the total standard deviation would be $\tilde{O}\left(\frac{h^{2}}{R}\sqrt{\frac{d}{L}}\cdot\sqrt{\frac{\kappa}{h}}\right)=\tilde{O}\left(\frac{h^{1.5}}{R}\sqrt{\frac{d}{m}}\right)$.
By choosing $R=\tilde{\Theta}\left(\frac{\sqrt{\kappa}}{\epsilon}\right)$,
it is enough to choose $h$ to be a constant to achieve less than
$\epsilon\sqrt{\frac{d}{m}}$ error, which shows that the algorithm
needs only $O\left(\frac{\kappa}{h}\log\frac{1}{\epsilon}\right)=O(\kappa\log\frac{1}{\epsilon})$
parallel steps. Appendix $\text{\ref{sec:Discretization-Error-of}}$
gives a partial proof of the guarantee of Algorithm $\text{\ref{al:lan2}}$.
The other part of the proof is similar to that in Algorithm $\text{\ref{al:lan}}$,
so we omit it here. 
\begin{thm}
\label{thm:main-1}Let $f$ be a function such that $0\prec m\cdot I_{d}\preceq\nabla^{2}f(x)\preceq L\cdot I_{d}$
for all $x\in\R^{d}$. Let $Y$ be a random point drawn from the density
proportional to $e^{-f}.$ Algorithm $\ref{al:lan2}$ can generate
a random point $X$ such that $W_{2}(X,Y)\leq\epsilon\sqrt{\frac{d}{m}}$
in $O(\kappa\log\frac{1}{\epsilon})$ parallel steps. Furthermore,
each iteration of Algorithm $\ref{al:lan2}$ involves computing $\tilde{\Theta}\left(\frac{\sqrt{\kappa}}{\epsilon}\right)$
of $\nabla f$s.
\end{thm}

\section{Numerical Experiments}

\begin{figure}[h]
\includegraphics[scale=0.33]{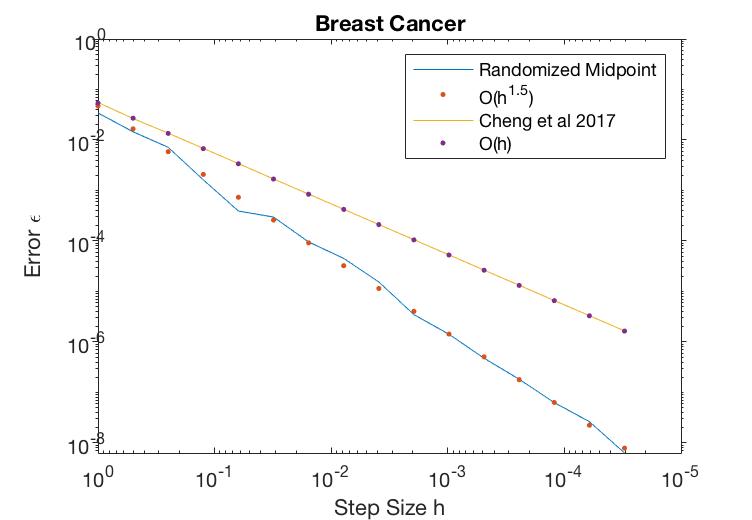}\includegraphics[scale=0.33]{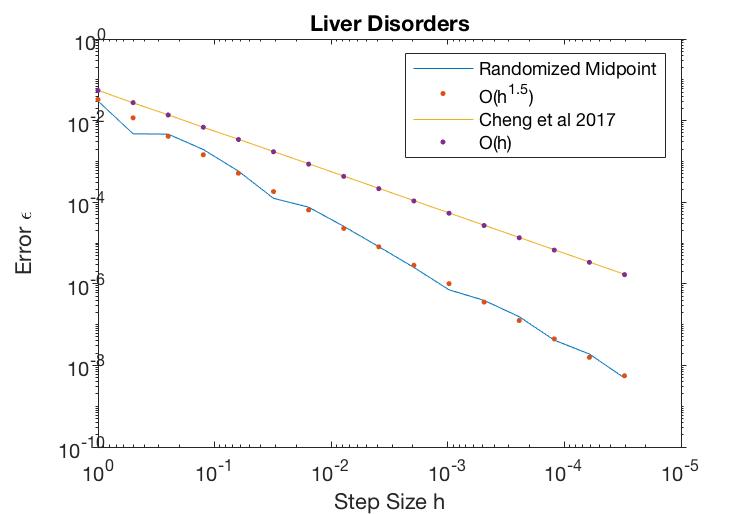}\caption{\label{fig:1}Error of random walks with different choice of step
size.}

\end{figure}

In this section, we compare the algorithm from our paper, randomized
midpoint method, with the one from \cite{cheng2017underdamped}. We
test the algorithms on the liver-disorders dataset and the breast-cancer
dataset from UCL machine learning \cite{Dua:2019}.\foreignlanguage{english}{
In both datasets, we observe a set of independent samples $\left\{ x_{i},y_{i}\right\} _{i=1}^{m}$,
where $y_{i}$ is the label, $x_{i}$ is the feature and $m$ is the
number of samples. We sample from the target distribution $p^{*}(\theta)\propto\exp\left(-f(\theta)\right),$
where 
\[
f(\theta)=\frac{\lambda}{2}\left\Vert \theta\right\Vert ^{2}+\frac{1}{m}\sum_{i=1}^{m}\log\left(\exp\left(-y_{i}x_{i}^{T}\theta\right)+1\right),
\]
for regularization parameters $\lambda$. We set $\lambda$ to be
$10^{-2}$ in our experiments. Figure $\text{\ref{fig:1}}$ shows
the error of randomized midpoint method and the algorithm from} \cite{cheng2017underdamped}\foreignlanguage{english}{
with different step size $h$. The error is measured by the $\ell_{2}$
distance to the true solution of (\ref{eq:ULD}) at time $N=5000$,
a time much greater than the mixing time of (\ref{eq:ULD}) for both
datasets. Our results show that the $\epsilon$ dependence analysis
of our algorithm and that of \cite{cheng2017underdamped} are both
tight.} However, we note that the logistic function is infinitely
differentiable, so there are methods of higher orders for this objective
such as the standard midpoint method and Runge--Kutta methods.

\bibliography{sampling.bib}
\bibliographystyle{plain}

\newpage{}

\appendix

\section{Brownian Motion Simulation}

\label{sec:Brownian-Motion-Simulation}

In this section, we introduce how $W_{1},$ $W_{2}$ and $W_{3}$
can be sampled. Let $\left\{ B_{t}\right\} _{t\in[0,h]}$ be the standard
$d$-dimensional Brownian motion on $t\in[0,h]$. In Algorithm $\text{\ref{al:lan}},$
$W_{1}=\int_{0}^{\alpha h}\left(1-e^{-2(\alpha h-s)}\right)\d B_{s}$,
$W_{2}=\int_{0}^{h}\left(1-e^{-2(h-s)}\right)\d B_{s}$ and $W_{3}=\int_{0}^{h}e^{-2(h-s)}\d B_{s}.$
We define $G_{1}=\int_{0}^{\alpha h}e^{2s}\d B_{s}$, $G_{2}=\int_{\alpha h}^{h}e^{2s}\d B_{s}$,
$H_{1}=\int_{0}^{\alpha h}\d B_{s}$ and $H_{2}=\int_{\alpha h}^{h}\d B_{s}$.
Then, $W_{1}=H_{1}-e^{-2\alpha h}G_{1}$, $W_{2}=(H_{1}+H_{2})-e^{-2h}(G_{1}+G_{2})$
and $W_{3}=e^{-2h}(G_{1}+G_{2})$. It is sufficient to sample $H_{1}$,
$H_{2}$, $G_{1}$ and $G_{2}$. We can show that $\left(G_{1},H_{1}\right)$
is independent of $\left(G_{2},H_{2}\right),$ and $\left(G_{1},H_{1}\right)$
and $\left(G_{2},H_{2}\right)$ both follow a $2d$-dimensional Gaussian
distribution, which can be easily sampled.
\begin{lem}
\label{lem:mean_var}Define $G_{1}=\int_{0}^{\alpha h}e^{2s}\d B_{s}$,
$G_{2}=\int_{\alpha h}^{h}e^{2s}\d B_{s}$, $H_{1}=\int_{0}^{\alpha h}\d B_{s}$
and $H_{2}=\int_{\alpha h}^{h}\d B_{s}$. Then, $\left(G_{1},H_{1}\right)$
is independent of $\left(G_{2},H_{2}\right)$ . Moreover, $\left(G_{1},H_{1}\right)$
and $\left(G_{2},H_{2}\right)$ both follow a $2d$-dimensional Gaussian
distribution with mean zero. Conditional on the choice of $\alpha$,
their covariance is given by
\begin{eqnarray*}
\text{\ensuremath{\E\left[\left(G_{1}-\E G_{1}\right)\left(H_{1}-\E H_{1}\right)^{T}\right]}} & = & \frac{1}{2}\left(e^{2\alpha h}-1\right)\cdot I_{d},\\
\E\left[\left(G_{1}-\E G_{1}\right)\left(G_{1}-\E G_{1}\right)^{T}\right] & = & \frac{1}{4}\left(e^{4\alpha h}-1\right)\cdot I_{d},\\
\E\left[\left(H_{1}-\E H_{1}\right)\left(H_{1}-\E H_{1}\right)^{T}\right] & = & \alpha h\cdot I_{d},\\
\E\left[\left(G_{2}-\E G_{2}\right)\left(H_{2}-\E H_{2}\right)^{T}\right] & = & \frac{1}{2}\left(e^{2h}-e^{2\alpha h}\right)\cdot I_{d},\\
\E\left[\left(G_{2}-\E G_{2}\right)\left(G_{2}-\E G_{2}\right)^{T}\right] & = & \frac{1}{4}\left(e^{4h}-e^{4\alpha h}\right)\cdot I_{d},\\
\E\left[\left(H_{2}-\E H_{2}\right)\left(H_{2}-\E H_{2}\right)^{T}\right] & = & \left(h-\alpha h\right)\cdot I_{d}.
\end{eqnarray*}
\end{lem}

\begin{proof}
By the definition of the standard Brownian motion, $\left(G_{1},H_{1}\right)$
is independent of $\left(G_{2},H_{2}\right)$ and $\left(G_{1},H_{1}\right)$
and $\left(G_{2},H_{2}\right)$ both have mean zero. Moreover,
\begin{align*}
\E\left[\left(G_{1}-\E G_{1}\right)\left(H_{1}-\E H_{1}\right)^{T}\right] & =\E\left[\left(\int_{0}^{\alpha h}e^{2s}\d B_{s}\right)\left(\int_{0}^{\alpha h}\d B_{s}\right)^{T}\right]=\int_{0}^{\alpha h}e^{2s}\d s\cdot I_{d}\\
 & =\frac{1}{2}\left(e^{2\alpha h}-1\right)\cdot I_{d},
\end{align*}
\begin{align*}
\E\left[\left(G_{1}-\E G_{1}\right)\left(G_{1}-\E G_{1}\right)^{T}\right] & =\E\left[\left(\int_{0}^{\alpha h}e^{2s}\d B_{s}\right)\left(\int_{0}^{\alpha h}e^{2s}\d B_{s}\right)^{T}\right]=\int_{0}^{\alpha h}e^{4s}\d s\cdot I_{d}\\
 & =\frac{1}{4}\left(e^{4\alpha h}-1\right)\cdot I_{d},
\end{align*}
and
\begin{eqnarray*}
\E\left[\left(H_{1}-\E H_{1}\right)\left(H_{1}-\E H_{1}\right)^{T}\right] & = & \alpha h\cdot I_{d}.
\end{eqnarray*}

Similarly,
\begin{align*}
\E\left[\left(G_{2}-\E G_{2}\right)\left(H_{2}-\E H_{2}\right)^{T}\right] & =\E\left[\left(\int_{\alpha h}^{h}e^{2s}\d B_{s}\right)\left(\int_{\alpha h}^{h}\d B_{s}\right)^{T}\right]=\int_{\alpha h}^{h}e^{2s}\d s\cdot I_{d}\\
 & =\frac{1}{2}\left(e^{2h}-e^{2\alpha h}\right)\cdot I_{d},
\end{align*}
\begin{align*}
\E\left[\left(G_{1}-\E G_{1}\right)\left(G_{1}-\E G_{1}\right)^{T}\right] & =\E\left[\left(\int_{\alpha h}^{h}e^{2s}\d B_{s}\right)\left(\int_{\alpha h}^{h}e^{2s}\d B_{s}\right)^{T}\right]=\int_{\alpha h}^{h}e^{4s}\d s\cdot I_{d}\\
 & =\frac{1}{4}\left(e^{4h}-e^{4\alpha h}\right)\cdot I_{d},
\end{align*}
and
\begin{eqnarray*}
\E\left[\left(H_{2}-\E H_{2}\right)\left(H_{2}-\E H_{2}\right)^{T}\right] & = & \left(h-\alpha h\right)\cdot I_{d}.
\end{eqnarray*}
\end{proof}

\section{Properties of the ULD and the Brownian motion}

\label{sec:Properties-of-the}

Here, we prove some properties of the ULD and the Brownian motion.
These properties are used in Appendices $\text{\ref{sec:Proof-of-Lemma}}$,
$\text{\ref{sec:Bounds-on-}}$, $\text{\ref{sec:Proof-of-Theorem}}$
and $\text{\ref{sec:Discretization-Error-of}}$ to prove the guarantee
of our algorithm.

\subsection{Properties of the ULD}
\begin{lem}
\label{lem:idealULD}Let $\left\{ x(t)\right\} _{t\in[0,h]}$ and
$\left\{ v(t)\right\} _{t\in[0,h]}$ be the solution to the underdamped
Langevin diffusion $\text{\ensuremath{\left(\ref{eq:ULD}\right)}}$
on $t\in[0,h]$. Assume that $h\leq\frac{1}{20}$ and $u=\frac{1}{L}$.
We have the following bounds.
\begin{eqnarray*}
\E\sup_{t\in[0,h]}\left\Vert v(t)\right\Vert ^{2} & \leq & O\left(\left\Vert v(0)\right\Vert ^{2}+u^{2}h^{2}\left\Vert \nabla f(x(0))\right\Vert ^{2}+udh\right),\\
\E\sup_{t\in[0,h]}\left\Vert \nabla f(x(t))\right\Vert ^{2} & \leq & O\left(\left\Vert \nabla f(x(0))\right\Vert ^{2}+L^{2}h^{2}\left\Vert v(0)\right\Vert ^{2}+Ldh^{3}\right),\\
\E\sup_{t\in[0,h]}\left\Vert x(0)-x(t)\right\Vert ^{2} & \leq & O\left(h^{2}\left\Vert v(0)\right\Vert ^{2}+u^{2}h^{4}\left\Vert \nabla f(x(0))\right\Vert ^{2}+udh^{3}\right)\text{,}
\end{eqnarray*}

and
\begin{eqnarray*}
-\E\inf_{t\in[0,h]}\left\Vert v(t)\right\Vert ^{2} & \leq & -\frac{1}{3}\left\Vert v(0)\right\Vert ^{2}+O\left(u^{2}h^{2}\left\Vert \nabla f(x(0))\right\Vert ^{2}+udh\right),\\
-\E\inf_{t\in[0,h]}\left\Vert \nabla f(x(t))\right\Vert ^{2} & \leq & -\frac{1}{3}\left\Vert \nabla f(x(0))\right\Vert ^{2}+O\left(h^{2}L^{2}\left\Vert v(0)\right\Vert ^{2}+Ldh^{3}\right)\text{.}
\end{eqnarray*}
\end{lem}

\begin{proof}
We first show the first three bounds. We can write $\E\sup_{t\in[0,h]}\left\Vert \nabla f(x(t))\right\Vert ^{2}$
as
\begin{eqnarray}
 &  & \E\sup_{t\in[0,h]}\left\Vert \nabla f(x(t))\right\Vert ^{2}\nonumber \\
 & \leq & 2\left\Vert \nabla f(x(0))\right\Vert ^{2}+2\E\sup_{t\in[0,h]}\left\Vert \nabla f(x(0))-\nabla f(x(t))\right\Vert ^{2}\nonumber \\
 & \leq & 2\left\Vert \nabla f(x(0))\right\Vert ^{2}+2L^{2}\E\sup_{t\in[0,h]}\left\Vert x(0)-x(t)\right\Vert ^{2},\label{eq:ideal_1}
\end{eqnarray}
where the first step follows by Young's inequality and the second
step follows by $\nabla f$ is $L$-Lipschitz. To bound $\E\sup_{t\in[0,h]}\left\Vert x(0)-x(t)\right\Vert ^{2}$,
\begin{eqnarray}
\E\sup_{t\in[0,h]}\left\Vert x(0)-x(t)\right\Vert ^{2} & = & \E\sup_{t\in[0,h]}\left\Vert \int_{0}^{t}v(s)\d s\right\Vert ^{2}\nonumber \\
 & \leq & \E\sup_{t\in[0,h]}t\int_{0}^{t}\left\Vert v(s)\right\Vert ^{2}\d s\nonumber \\
 & \leq & h^{2}\E\sup_{t\in[0,h]}\left\Vert v(t)\right\Vert ^{2},\label{eq:ideal_2}
\end{eqnarray}
where the first step follows by the definition of $x$ and the second
follows by the Cauchy-Schwarz inequality. To bound $\E\sup_{t\in[0,h]}\left\Vert v(t)\right\Vert ^{2},$
\begin{eqnarray}
\E\sup_{t\in[0,h]}\left\Vert v(t)\right\Vert ^{2} & = & \E\sup_{t\in[0,h]}\left\Vert v(0)e^{-2t}-u\int_{0}^{t}e^{-2(t-s)}\nabla f(x(s))\d s+2\sqrt{u}\int_{0}^{t}e^{-2(t-s)}\d B_{s}\right\Vert ^{2}\nonumber \\
 & \leq & 3\left\Vert v(0)\right\Vert ^{2}+3u^{2}h^{2}\E\sup_{t\in[0,h]}\left\Vert \nabla f(x(t))\right\Vert ^{2}+12u\E\sup_{t\in[0,h]}\left\Vert \int_{0}^{t}e^{-2(t-s)}\d B_{s}\right\Vert ^{2}\nonumber \\
 & \leq & 3\left\Vert v(0)\right\Vert ^{2}+3u^{2}h^{2}\E\sup_{t\in[0,h]}\left\Vert \nabla f(x(t))\right\Vert ^{2}+60udh,\label{eq:ideal_3}
\end{eqnarray}
where the first step follows by the definition of ULD, the second
step follows by the inequality $(a+b+c)^{2}\leq3a^{2}+3b^{2}+3c^{2}$
and the third step follows by Lemma $\text{\ref{lem:Brownian}}.$
Then, combining $\text{\ensuremath{\left(\ref{eq:ideal_1}\right)}}$,
$\left(\text{\ref{eq:ideal_2}}\right)$ and $\left(\text{\ref{eq:ideal_3}}\right)$,
we have
\begin{align*}
\E\sup_{t\in[0,h]}\left\Vert \nabla f(x(t))\right\Vert ^{2} & \leq2\left\Vert \nabla f(x(0))\right\Vert ^{2}+2L^{2}\E\sup_{t\in[0,h]}\left\Vert x(0)-x(t)\right\Vert ^{2}\\
 & \leq2\left\Vert \nabla f(x(0))\right\Vert ^{2}+2L^{2}h^{2}\E\sup_{t\in[0,h]}\left\Vert v(t)\right\Vert ^{2}\\
 & \leq2\left\Vert \nabla f(x(0))\right\Vert ^{2}+6h^{4}\E\sup_{t\in[0,h]}\left\Vert \nabla f(x(t))\right\Vert ^{2}+6L^{2}h^{2}\left\Vert v(0)\right\Vert ^{2}+120Ldh^{3}.
\end{align*}
Since $6h^{4}\leq\frac{1}{4}$, 
\begin{eqnarray}
\E\sup_{t\in[0,h]}\left\Vert \nabla f(x(t))\right\Vert ^{2} & \leq & 3\left\Vert \nabla f(x(0))\right\Vert ^{2}+8L^{2}h^{2}\left\Vert v(0)\right\Vert ^{2}+160Ldh^{3}\nonumber \\
 & \leq & O\left(\left\Vert \nabla f(x(0))\right\Vert ^{2}+L^{2}h^{2}\left\Vert v(0)\right\Vert ^{2}+Ldh^{3}\right).\label{eq:ideal_4}
\end{eqnarray}
By $\text{\ensuremath{\left(\ref{eq:ideal_3}\right)}}$ and $\text{\ensuremath{\left(\ref{eq:ideal_4}\right)}}$,
\begin{eqnarray*}
\E\sup_{t\in[0,h]}\left\Vert v(t)\right\Vert ^{2} & \leq & 3\left\Vert v(0)\right\Vert ^{2}+3u^{2}h^{2}\E\sup_{t\in[0,h]}\left\Vert \nabla f(x(t))\right\Vert ^{2}+60udh\\
 & \leq & 3\left\Vert v(0)\right\Vert ^{2}+3u^{2}h^{2}\ensuremath{\cdot}O\left(\left\Vert \nabla f(x(0))\right\Vert ^{2}+L^{2}h^{2}\left\Vert v(0)\right\Vert ^{2}+Ldh^{3}\right)+60udh\\
 & \leq & O\left(\left\Vert v(0)\right\Vert ^{2}+u^{2}h^{2}\left\Vert \nabla f(x(0))\right\Vert ^{2}+udh\right).
\end{eqnarray*}
where the last step follows by $h$ is small.

By $\left(\text{\ref{eq:ideal_2}}\right)$ and $\left(\text{\ref{eq:ideal_4}}\right)$,
\begin{eqnarray}
\E\sup_{t\in[0,h]}\left\Vert x(0)-x(t)\right\Vert ^{2} & \leq & h^{2}\E\sup_{t\in[0,h]}\left\Vert v(t)\right\Vert ^{2}\nonumber \\
 & \leq & O\left(h^{2}\left\Vert v(0)\right\Vert ^{2}+u^{2}h^{4}\left\Vert \nabla f(x(0))\right\Vert ^{2}+udh^{3}\right).\label{eq:ideal_5}
\end{eqnarray}

To prove the fourth claim, 
\begin{eqnarray*}
 &  & \inf_{t\in[0,h]}\left\Vert v(t)\right\Vert ^{2}\\
 & = & \inf_{t\in[0,h]}\left\Vert v(0)e^{-2t}-u\int_{0}^{t}e^{-2(t-s)}\nabla f(x(s))\d s+2\sqrt{u}\int_{0}^{t}e^{-2(t-s)}\d B_{s}\right\Vert ^{2}\\
 & \geq & \inf_{t\in[0,h]}\Bigg[e^{-4t}\left\Vert v(0)\right\Vert ^{2}-2e^{-2t}v(0)^{T}\left(u\int_{0}^{t}e^{-2(t-s)}\nabla f(x(s))\d s\right)\\
 &  & +2e^{-2t}v(0)^{T}\left(2\sqrt{u}\int_{0}^{t}e^{-2(t-s)}\d B_{s}\right)\Bigg]\\
 & \geq & \inf_{t\in[0,h]}\Bigg[e^{-4t}\left\Vert v(0)\right\Vert ^{2}-\frac{1}{2}e^{-4t}\left\Vert v(0)\right\Vert ^{2}-4\left\Vert u\int_{0}^{t}e^{-2(t-s)}\nabla f(x(s))\d s\right\Vert ^{2}\\
 &  & -4\left\Vert 2\sqrt{u}\int_{0}^{t}e^{-2(t-s)}\d B_{s}\right\Vert ^{2}\Bigg]\\
 & \geq & \inf_{t\in[0,h]}\left[\frac{1}{2}(1-4h)\left\Vert v(0)\right\Vert ^{2}-4u^{2}h^{2}\sup_{s\in[0,t]}\left\Vert \nabla f(x(s))\right\Vert ^{2}-16u\left\Vert \int_{0}^{t}e^{-2(t-s)}\d B_{s}\right\Vert ^{2}\right]\\
 & \geq & \frac{1}{2}(1-4h)\left\Vert v(0)\right\Vert ^{2}-4u^{2}h^{2}\sup_{t\in[0,h]}\left\Vert \nabla f(x(t))\right\Vert ^{2}-16u\sup_{t\in[0,h]}\left\Vert \int_{0}^{t}e^{-2(t-s)}\d B_{s}\right\Vert ^{2},
\end{eqnarray*}
where the first step follows by the definition of $v$ , the second
step follows by the inequality $(a+b+c)^{2}\geq a^{2}+2a(b+c)$, the
third step follows by the inequality $2ab\leq a^{2}+b^{2}$ , the
fourth step follows by $e^{-4t}\geq1-4t$, and the last step follows
by $h$ is small.

Then, by $\text{\ensuremath{\left(\ref{eq:ideal_4}\right)}}$ and
Lemma $\text{\ref{lem:Brownian}},$
\begin{eqnarray*}
-\E\inf_{t\in[0,h]}\left\Vert v(t)\right\Vert ^{2} & \leq & -\frac{1}{3}\left\Vert v(0)\right\Vert ^{2}+O\left(u^{2}h^{2}\left\Vert \nabla f(x(0))\right\Vert ^{2}+udh\right).
\end{eqnarray*}

To show the lower bound on $\E\inf_{t\in[0,h]}\left\Vert \nabla f(x(t))\right\Vert ^{2}$,
notice that 
\begin{eqnarray*}
\E\inf_{t\in[0,h]}\left\Vert \nabla f(x(t))\right\Vert ^{2} & \geq & \frac{1}{2}\left\Vert \nabla f(x(0))\right\Vert ^{2}-\E\sup_{t\in[0,h]}\left\Vert \nabla f(x(t))-\nabla f(x(0))\right\Vert ^{2}\\
 & \geq & \frac{1}{2}\left\Vert \nabla f(x(0))\right\Vert ^{2}-L^{2}\E\sup_{t\in[0,h]}\left\Vert x(t)-x(0)\right\Vert ^{2}.
\end{eqnarray*}
Then, by$\text{\ensuremath{\left(\ref{eq:ideal_5}\right)}}$ and $h\leq\frac{1}{20}$,
\begin{eqnarray*}
-\E\inf_{t\in[0,h]}\left\Vert \nabla f(x(t))\right\Vert  & \leq & -\frac{1}{3}\left\Vert \nabla f(x(0))\right\Vert ^{2}+O\left(h^{2}L^{2}\left\Vert v(0)\right\Vert ^{2}+Ldh^{3}\right)\text{.}
\end{eqnarray*}
\end{proof}

\subsection{Properties of the Brownian Motion}
\begin{lem}[Doob's maximal inequality \cite{doob1953stochastic}]
\label{lem:doob}Suppose $\{X(t):t\geq0\}$ is a continuous martingale.
Then, for any $t\geq0$,
\begin{eqnarray*}
\E\left[\sup_{0\leq s\leq t}\left|X(s)\right|^{2}\right] & \leq & 4\E\left[\left|X(t)\right|^{2}\right].
\end{eqnarray*}
\end{lem}

Using the Doob's maximal inequality, we can show the following lemma.
\begin{lem}
\label{lem:Brownian}For $d$-dimensional Brownian motion $B_{t}$
on $t\in[0,h]$, assuming $h\leq\frac{1}{10},$
\[
\E\left[\sup_{0\leq t\leq h}\left\Vert B(t)\right\Vert ^{2}\right]\leq4dh,\text{ and }\E\left[\sup_{0\leq t\leq h}\left\Vert \int_{0}^{t}e^{-2(t-s)}dB_{s}\right\Vert ^{2}\right]\leq5dh.
\]
\end{lem}

\begin{proof}
To show the first inequality,
\begin{eqnarray*}
\E\left[\sup_{0\leq t\leq h}\left\Vert B(t)\right\Vert ^{2}\right] & \leq & \sum_{i=1}^{d}\E\left[\sup_{0\leq t\leq h}\left|B_{i}(t)\right|^{2}\right]\\
 & \leq & 4d\E\left[\left|B_{i}(h)\right|^{2}\right]\\
 & = & 4dh,
\end{eqnarray*}
where the second step follows by Lemma $\text{\ref{lem:doob}}$. To
show the second inequality,
\begin{eqnarray*}
\E\left[\sup_{0\leq t\leq h}\left\Vert \int_{0}^{t}e^{-2(t-s)}dB_{s}\right\Vert ^{2}\right] & \leq & \E\left[\sup_{0\leq t\leq h}e^{-4t}\left\Vert \int_{0}^{t}e^{2s}dB_{s}\right\Vert ^{2}\right]\\
 & \leq & \E\left[\sup_{0\leq t\leq h}\left\Vert \int_{0}^{t}e^{2s}dB_{s}\right\Vert ^{2}\right]\\
 & \leq & \sum_{i=1}^{d}\E\left[\sup_{0\leq t\leq h}\left|\int_{0}^{t}e^{2s}dB_{s,i}\right|^{2}\right]\\
 & \leq & 4\sum_{i=1}^{d}\E\left[\left|\int_{0}^{h}e^{2s}dB_{s,i}\right|^{2}\right]\\
 & = & 4\sum_{i=1}^{d}\int_{0}^{h}e^{4s}ds\\
 & \leq & 5dh,
\end{eqnarray*}
where the second step follows by $e^{-4t}\leq1$, the fourth step
follows by Lemma $\text{\ref{lem:doob}}$ and the last inequality
follows by $\int_{0}^{h}e^{4s}ds\leq\frac{5}{4}h$ for $h\leq\frac{1}{10}.$
\end{proof}

\section{Discretization Error of Algorithm $\text{\ref{al:lan}}$}

\label{sec:Proof-of-Lemma}

In this section, we bound the discretization error of Algorithm $\text{\ref{al:lan}}$
in each iteration. In order to prove Lemma $\text{\ref{lem:error}}$,
we first prove Lemma \ref{lem:grad_f}, stated next.
\begin{lem}
\label{lem:grad_f} Let $\alpha$ be the random number chosen in iteration
$n$. Let $x_{n+\frac{1}{2}}$ be the intermediate value computed
in iteration $n$ of Algorithm $\text{\ref{al:lan}}$. Let $\left\{ x_{n}^{*}(t)\right\} _{t\in[0,h]}$
be the ideal underdamped Langevin diffusion starting from $x_{n}^{*}(0)=x_{n}$
coupled through a shared Brownian motion with $x_{n+\frac{1}{2}}.$
Assume that $h\leq\frac{1}{20}$. Then, 
\begin{eqnarray*}
\E\left\Vert \nabla f(x_{n+\frac{1}{2}})-\nabla f(x_{n}^{*}(\alpha h))\right\Vert ^{2} & \leq & O\left(h^{6}L^{2}\left\Vert v_{n}\right\Vert ^{2}+h^{8}\left\Vert \nabla f(x_{n})\right\Vert ^{2}+Ldh^{7}\right).
\end{eqnarray*}
\end{lem}

\begin{proof}
We have the bound
\begin{eqnarray*}
 &  & \E\left\Vert \nabla f(x_{n+\frac{1}{2}})-\nabla f(x_{n}^{*}(\alpha h))\right\Vert ^{2}\\
 & \leq & L^{2}\E\left\Vert x_{n+\frac{1}{2}}-x_{n}^{*}(\alpha h)\right\Vert ^{2}\\
 & = & L^{2}\E\left\Vert \frac{1}{2}u\int_{0}^{\alpha h}\left(1-e^{-2(\alpha h-s)}\right)\left(\nabla f(x_{n}^{*}(0))-\nabla f(x_{n}^{*}(s))\right)\d s\right\Vert ^{2}\\
 & \leq & \frac{1}{4}\E\left[\int_{0}^{\alpha h}\left(1-e^{-2(\alpha h-s)}\right)^{2}\d s\cdot\alpha h\cdot\left(\sup_{t\in[0,h]}\left\Vert \nabla f(x_{n}^{*}(0))-\nabla f(x_{n}^{*}(t))\right\Vert ^{2}\right)\right]\\
 & \leq & h^{4}\E\sup_{t\in[0,h]}\left\Vert \nabla f(x_{n}^{*}(0))-\nabla f(x_{n}^{*}(t))\right\Vert ^{2}\\
 & \leq & L^{2}h^{4}\E\sup_{t\in[0,h]}\left\Vert x_{n}^{*}(0)-x_{n}^{*}(t)\right\Vert ^{2}\\
 & \leq & O\left(h^{6}L^{2}\left\Vert v_{n}\right\Vert ^{2}+h^{8}\left\Vert \nabla f(x_{n})\right\Vert ^{2}+Ldh^{7}\right),
\end{eqnarray*}
where the first and the fifth step follows by $\nabla f$ is $L$-Lipschitz,
the third step follows by Cauchy-Schwarz inequality, the fourth step
follows by $1-e^{-2(\alpha h-t)}\leq2h$ and the last step follows
by Lemma $\text{\ref{lem:idealULD}.}$
\end{proof}
Now, we are ready to prove Lemma $\text{\ref{lem:error}}$.
\begin{proof}
To show the first claim,
\begin{eqnarray*}
 &  & \left\Vert \E_{\alpha}x_{n+1}-x_{n}^{*}(h)\right\Vert ^{2}\\
 & = & \left\Vert \E_{\alpha}\frac{1}{2}uh\left(1-e^{-2(h-\alpha h)}\right)\nabla f(x_{n+\frac{1}{2}})-\frac{1}{2}u\int_{0}^{h}\left(1-e^{-2(h-s)}\right)\nabla f(x_{n}^{*}(s))\d s\right\Vert ^{2}\\
 & \leq & \frac{1}{2}\E_{\alpha}\left\Vert uh\left(1-e^{-2(h-\alpha h)}\right)\nabla f(x_{n+\frac{1}{2}})-uh\left(1-e^{-2(h-\alpha h)}\right)\nabla f(x_{n}^{*}(\alpha h))\right\Vert ^{2}\\
 &  & +\frac{1}{2}\left\Vert \E_{\alpha}uh\left(1-e^{-2(h-\alpha h)}\right)\nabla f(x_{n}^{*}(\alpha h))-u\int_{0}^{h}\left(1-e^{-2(h-s)}\right)\nabla f(x_{n}^{*}(s))\d s\right\Vert ^{2}\\
 & \leq & \frac{1}{2}u^{2}h^{2}\E_{\alpha}\left[\left(1-e^{-2(h-\alpha h)}\right)^{2}\left\Vert \nabla f(x_{n+\frac{1}{2}})-\nabla f(x_{n}^{*}(\alpha h))\right\Vert ^{2}\right]+0\\
 & \leq & 2u^{2}h^{4}\E_{\alpha}\left\Vert \nabla f(x_{n+\frac{1}{2}})-\nabla f(x_{n}^{*}(\alpha h))\right\Vert ^{2},
\end{eqnarray*}
where the first step follows by the definition of $x_{n+1}$, the
second step follows by Young's inequality, the third step follows
by 
\[
\E_{\alpha}h\left(1-e^{-2(h-\alpha h)}\right)\nabla f(x_{n}^{*}(\alpha h))=\int_{0}^{h}\left(1-e^{-2(h-s)}\right)\nabla f(x_{n}^{*}(s))\d s,
\]
and the fourth step follows by $1-e^{-2(h-\alpha h)}\leq2h$ . By
Lemma $\text{\ref{lem:grad_f}}$,
\begin{eqnarray*}
\E\left\Vert \E_{\alpha}x_{n+1}-x_{n}^{*}(h)\right\Vert ^{2} & \leq & O\left(h^{10}\left\Vert v_{n}\right\Vert ^{2}+u^{2}h^{12}\left\Vert \nabla f(x_{n})\right\Vert ^{2}+udh^{11}\right).
\end{eqnarray*}

To show the second claim,
\begin{eqnarray*}
 &  & \E\left\Vert x_{n+1}-x_{n}^{*}(h)\right\Vert ^{2}\\
 & \leq & \frac{3}{4}\E\left\Vert uh\left(1-e^{-2(h-\alpha h)}\right)\nabla f(x_{n+\frac{1}{2}})-uh\left(1-e^{-2(h-\alpha h)}\right)\nabla f(x_{n}^{*}(\alpha h))\right\Vert ^{2}\\
 &  & +\frac{3}{4}\E\left\Vert uh\left(1-e^{-2(h-\alpha h)}\right)\nabla f(x_{n}^{*}(\alpha h))-u\int_{0}^{h}\left(1-e^{-2(h-\alpha h)}\right)\nabla f(x_{n}^{*}(s))\d s\right\Vert ^{2}\\
 &  & +\frac{3}{4}\E\left\Vert u\int_{0}^{h}\left(1-e^{-2(h-\alpha h)}\right)\nabla f(x_{n}^{*}(s))\d s-u\int_{0}^{h}\left(1-e^{-2(h-s)}\right)\nabla f(x_{n}^{*}(s))\d s\right\Vert ^{2},
\end{eqnarray*}
which follows by definition and Young's inequality. To bound the second
term,
\begin{eqnarray}
 &  & \left\Vert uh\left(1-e^{-2(h-\alpha h)}\right)\nabla f(x_{n}^{*}(\alpha h))-u\int_{0}^{h}\left(1-e^{-2(h-\alpha h)}\right)\nabla f(x_{n}^{*}(s))\d s\right\Vert ^{2}\nonumber \\
 & = & \left\Vert u\int_{0}^{h}\left(1-e^{-2(h-\alpha h)}\right)\left(\nabla f(x_{n}^{*}(\alpha h))-\nabla f(x_{n}^{*}(s))\right)\d s\right\Vert ^{2}\nonumber \\
 & \leq & u^{2}\int_{0}^{h}\left(1-e^{-2(h-\alpha h)}\right)^{2}\d s\cdot\sup_{t\in[0,h]}\left\Vert \nabla f(x_{n}^{*}(\alpha h))-\nabla f(x_{n}^{*}(t))\right\Vert ^{2}\cdot h\nonumber \\
 & \leq & 4u^{2}h^{4}\sup_{t\in[0,h]}\left\Vert \nabla f(x_{n}^{*}(\alpha h))-\nabla f(x_{n}^{*}(t))\right\Vert ^{2}\nonumber \\
 & \leq & 16h^{4}\sup_{t\in[0,h]}\left\Vert x_{n}^{*}(0)-x_{n}^{*}(t)\right\Vert ^{2}\label{eq:error_1}
\end{eqnarray}
where the second step follows by the Cauchy-Schwarz inequality. The
third term satisfies
\begin{eqnarray}
 &  & \left\Vert u\int_{0}^{h}\left(1-e^{-2(h-\alpha h)}\right)\nabla f(x_{n}^{*}(s))\d s-u\int_{0}^{h}\left(1-e^{-2(h-s)}\right)\nabla f(x_{n}^{*}(s))\d s\right\Vert ^{2}\nonumber \\
 & = & u^{2}\left\Vert \int_{0}^{h}\left(e^{-2(h-s)}-e^{-2(h-\alpha h)}\right)\nabla f(x_{n}^{*}(s))\d s\right\Vert ^{2}\nonumber \\
 & \leq & 4u^{2}h^{4}\sup_{t\in[0,h]}\left\Vert \nabla f(x_{n}^{*}(t))\right\Vert ^{2},\label{eq:error_2}
\end{eqnarray}
where the second step follows by the Cauchy Schwarz inequality and
$\left|e^{-2(h-s)}-e^{-2(h-\alpha h)}\right|\leq2h$. Thus, 
\begin{eqnarray*}
 &  & \E\left\Vert x_{n+1}-x_{n}^{*}(h)\right\Vert ^{2}\\
 & \leq & 3u^{2}h^{4}\E\left\Vert \nabla f(x_{n+\frac{1}{2}})-\nabla f(x_{n}^{*}(\alpha h))\right\Vert ^{2}+12h^{4}\E\sup_{t\in[0,h]}\left\Vert x_{n}^{*}(0)-x_{n}^{*}(t)\right\Vert ^{2}\\
 &  & +3u^{2}h^{4}\E\sup_{t\in[0,h]}\left\Vert \nabla f(x_{n}^{*}(t))\right\Vert ^{2}\\
 & \leq & 3h^{4}\cdot O\left(h^{6}\left\Vert v_{n}\right\Vert ^{2}+h^{8}u^{2}\left\Vert \nabla f(x_{n})\right\Vert ^{2}+udh^{7}\right)\\
 &  & +12h^{4}\cdot O\left(h^{2}\left\Vert v_{n}\right\Vert ^{2}+u^{2}h^{4}\left\Vert \nabla f(x_{n})\right\Vert ^{2}+udh^{3}\right)\\
 &  & +3u^{2}h^{4}\cdot O\left(\left\Vert \nabla f(x_{n})\right\Vert ^{2}+L^{2}h^{2}\left\Vert v_{n}\right\Vert ^{2}+Mdh^{3}\right)\\
 & \leq & O\left(h^{6}\left\Vert v_{n}\right\Vert ^{2}+u^{2}h^{4}\left\Vert \nabla f(x_{n})\right\Vert ^{2}+udh^{7}\right).
\end{eqnarray*}
where the first step follows by $\left(\text{\ref{eq:error_1}}\right)$
and $\left(\text{\ref{eq:error_2}}\right)$, the second step follows
by Lemma $\text{\ref{lem:idealULD}}$ and Lemma $\text{\ref{lem:grad_f}}$,
and the last inequality follows by $h\leq1$. 

To show the third claim,
\begin{eqnarray*}
\E\left\Vert \E_{\alpha}v_{n+1}-v_{n}^{*}(h)\right\Vert ^{2} & = & \E\left\Vert \E_{\alpha}uhe^{-2(h-\alpha h)}\nabla f(x_{n+\frac{1}{2}})-u\int_{0}^{h}e^{-2(h-s)}\nabla f(x_{n}^{*}(s))\d s\right\Vert ^{2}\\
 & \leq & 2\E\left\Vert uhe^{-2(h-\alpha h)}\nabla f(x_{n+\frac{1}{2}})-uhe^{-2(h-\alpha h)}\nabla f(x_{n}^{*}(\alpha h))\right\Vert ^{2}\\
 &  & +2\E\left\Vert \E_{\alpha}uhe^{-2(h-\alpha h)}\nabla f(x_{n}^{*}(\alpha h))-u\int_{0}^{h}e^{-2(h-s)}\nabla f(x_{n}^{*}(s))\d s\right\Vert ^{2}\\
 & \leq & 2u^{2}h^{2}\E\left\Vert \nabla f(x_{n+\frac{1}{2}})-\nabla f(x_{n}^{*}(\alpha h))\right\Vert ^{2}+0\\
 & \leq & O\left(h^{8}\left\Vert v_{n}\right\Vert ^{2}+u^{2}h^{10}\left\Vert \nabla f(x_{n})\right\Vert ^{2}+udh^{9}\right),
\end{eqnarray*}
where the first step follows by Young's inequality, the second step
follows by 
\begin{eqnarray*}
\E_{\alpha}uhe^{-2(h-\alpha h)}\nabla f(x_{n}^{*}(\alpha h)) & = & u\int_{0}^{h}e^{-2(h-t)}\nabla f(x_{n}^{*}(t))\d t,
\end{eqnarray*}
and $e^{-2(h-\alpha h)}\leq1$, and the third step follows by Lemma
$\text{\ref{lem:grad_f}}.$

To show the last claim,
\begin{eqnarray*}
 &  & \E\left\Vert v_{n+1}-v_{n}^{*}(h)\right\Vert ^{2}\\
 & = & \E\left\Vert uhe^{-2(h-\alpha h)}\nabla f(x_{n+\frac{1}{2}})-u\int_{0}^{h}e^{-2(h-s)}\nabla f(x^{*}(s))\d s\right\Vert ^{2}\\
 & \leq & 3\E\left\Vert uhe^{-2(h-\alpha h)}\nabla f(x_{n+\frac{1}{2}})-uhe^{-2(h-\alpha h)}\nabla f(x^{*}(\alpha h))\right\Vert ^{2}\\
 &  & +3\E\left\Vert u\int_{0}^{h}e^{-2(h-\alpha h)}\nabla f(x_{n}^{*}(\alpha h))\d t-u\int_{0}^{h}e^{-2(h-\alpha h)}\nabla f(x_{n}^{*}(s))\d s\right\Vert ^{2}\\
 &  & +3\E\left\Vert u\int_{0}^{h}e^{-2(h-\alpha h)}\nabla f(x_{n}^{*}(s))\d s-u\int_{0}^{h}e^{-2(h-s)}\nabla f(x_{n}^{*}(s))\d s\right\Vert ^{2}\\
 & \leq & 3u^{2}h^{2}\E\left\Vert \nabla f(x_{n+\frac{1}{2}})-\nabla f(x_{n}^{*}(\alpha h))\right\Vert ^{2}+3h^{2}\E\sup_{t\in[0,h]}\left\Vert x_{n}^{*}(\alpha h)-x_{n}^{*}(t)\right\Vert ^{2}\\
 &  & +12u^{2}h^{4}\E\sup_{t\in[0,h]}\left\Vert \nabla f(x_{n}^{*}(t))\right\Vert ^{2}\\
 & \leq & 3u^{2}h^{2}\cdot O\left(h^{6}L^{2}\left\Vert v_{n}\right\Vert ^{2}+h^{8}\left\Vert \nabla f(x_{n})\right\Vert ^{2}+Ldh^{7}\right)\\
 &  & +3h^{2}\cdot O\left(h^{2}\left\Vert v_{n}\right\Vert ^{2}+u^{2}h^{4}\left\Vert \nabla f(x_{n})\right\Vert ^{2}+udh^{3}\right)\\
 &  & +12u^{2}h^{4}\cdot O\left(\left\Vert \nabla f(x_{n})\right\Vert ^{2}+L^{2}h^{2}\left\Vert v_{n}\right\Vert ^{2}+Ldh^{3}\right)\\
 & \leq & O\left(h^{4}\left\Vert v_{n}\right\Vert ^{2}+u^{2}h^{4}\left\Vert \nabla f(x_{n})\right\Vert ^{2}+udh^{5}\right),
\end{eqnarray*}
where the first step follows by the definition, the second step follows
by Young's inequality, the third follows by $e^{-2(h-\alpha h)}-e^{-2(h-s)}\leq2h$
, the fourth step follows by Lemma $\text{\ref{lem:grad_f}}$ and
Lemma $\text{\ref{lem:idealULD}}$ and the last inequality follows
by $h\leq1.$
\end{proof}

\section{Bounds on $\left\Vert \nabla f(x)\right\Vert $ and $\left\Vert v\right\Vert $}

\label{sec:Bounds-on-}

In this section, we bound the sum of $\left\Vert \nabla f(x_{n})\right\Vert ^{2}$
and $\left\Vert v_{n}\right\Vert ^{2}$ over all iterations $n$,
$\sum_{n=0}^{N-1}\E\left\Vert \nabla f(x_{n})\right\Vert ^{2}$ and
$\sum_{n=0}^{N-1}\E\left\Vert v_{n}\right\Vert ^{2}$. In Appendix
$\text{\ref{sec:Proof-of-Theorem}}$, we use the results in this appendix
together with Lemma $\text{\ref{lem:error}}$ to prove the guarantee
of our algorithm. 
\begin{lem}
\label{lem:f_error}Assume $h\leq\frac{1}{20}$. For each iteration
$n$, let $x_{n}$ be the starting point of iteration $n$ of Algorithm
$\text{\ref{al:lan}}$. Let $\left\{ v_{n}(t),x_{n}(t)\right\} _{t\in[0,h]}$
be the solution of the exact underdamped Langevin diffusion starting
from $\left(v_{n},x_{n}\right)$ . Let $\E_{\alpha}$ be the expectation
over the random choice of $\alpha$ in iteration $n$. Then, the difference
between the value of $f$ on the starting point of iteration $n+1$,
$x_{n+1}$, and that of $x_{n}(h)$ satisfies
\begin{eqnarray*}
\E f(x_{n+1}(0))-f(x_{n}(h)) & \leq & O\left(uh^{3}\left\Vert \nabla f(x_{n}(0))\right\Vert ^{2}+Lh^{5}\left\Vert v_{n}(0)\right\Vert ^{2}+dh^{6}\right).
\end{eqnarray*}
\end{lem}

\begin{proof}
We first consider the expectation over the choice of $\alpha$ in
iteration $n$,
\begin{eqnarray*}
 &  & \E_{\alpha}f(x_{n+1}(0))\\
 & \leq & f(x_{n}(h))+\nabla f(x_{n}(h))^{T}\left(\E_{\alpha}x_{n+1}(0)-x_{n}(h)\right)+\frac{L}{2}\E_{\alpha}\left\Vert x_{n+1}(0)-x_{n}(h)\right\Vert ^{2}\\
 & \leq & f(x_{n}(h))+\left\Vert \nabla f(x_{n}(h))\right\Vert \left\Vert \E_{\alpha}x_{n+1}(0)-x_{n}(h)\right\Vert +\frac{L}{2}\E_{\alpha}\left\Vert x_{n+1}(0)-x_{n}(h)\right\Vert ^{2}\\
 & \leq & f(x_{n}(h))+uh^{3}\left\Vert \nabla f(x_{n}(h))\right\Vert ^{2}+\frac{L}{h^{3}}\left\Vert \E_{\alpha}x_{n+1}(0)-x_{n}(h)\right\Vert ^{2}+\frac{L}{2}\E_{\alpha}\left\Vert x_{n+1}(0)-x_{n}(h)\right\Vert ^{2},
\end{eqnarray*}
where the first step follows by $\nabla f$ is $L$-Lipschitz, the
second step follows by Cauchy-Schwarz inequality and the third step
follows by Young's inequality. By Lemma $\text{\ref{lem:error}}$
and Lemma $\text{\ref{lem:idealULD}},$ 
\begin{eqnarray*}
\E f(x_{n+1}(0)) & \leq & \E f(x_{n}(h))+uh^{3}\E\left\Vert \nabla f(x_{n}(h))\right\Vert ^{2}+\frac{L}{h^{3}}\E\left\Vert \E_{\alpha}x_{n+1}(0)-x_{n}(h)\right\Vert ^{2}\\
 &  & +\frac{L}{2}\E\left\Vert x_{n+1}(0)-x_{n}(h)\right\Vert ^{2}\\
 & \leq & \E f(x_{n}(h))+\E uh^{3}\cdot O\left(\left\Vert \nabla f(x_{n}(0))\right\Vert ^{2}+L^{2}h^{2}\left\Vert v_{n}(0)\right\Vert ^{2}+Ldh^{3}\right)\\
 &  & +\E\frac{L}{h^{3}}\cdot O\left(h^{10}\left\Vert v_{n}(0)\right\Vert ^{2}+u^{2}h^{12}\left\Vert \nabla f(x_{n}(0))\right\Vert ^{2}+udh^{11}\right)\\
 &  & +\E\frac{L}{2}\cdot O\left(h^{6}\left\Vert v_{n}(0)\right\Vert ^{2}+h^{4}u^{2}\left\Vert \nabla f(x_{n}(0))\right\Vert ^{2}+udh^{7}\right)\\
 & \leq & \E f(x_{n}(h))+O\left(uh^{3}\E\left\Vert \nabla f(x_{n}(0))\right\Vert ^{2}+Lh^{5}\E\left\Vert v_{n}(0)\right\Vert ^{2}+dh^{6}\right).
\end{eqnarray*}
where the second step follows by Lemma $\text{\ref{lem:error}}$ and
Lemma $\text{\ref{lem:idealULD}}$, and the last step follows by $h\leq\frac{1}{20}$.
\end{proof}
\begin{lem}
\label{lem:sum_v}Assume $h$ is smaller than some given constant.
For each iteration $n=0,...,N-1$, let $\left(v_{n},x_{n}\right)$
be the starting point of Algorithm $\text{\ref{al:lan}}$ in iteration
$n$. Then, 
\begin{eqnarray*}
\sum_{n=0}^{N-1}\E\left\Vert v_{n}\right\Vert ^{2} & \leq & O\left(u^{2}h\sum_{n=0}^{N-1}\E\left\Vert \nabla f(x_{n})\right\Vert ^{2}+Nud\right).
\end{eqnarray*}
\end{lem}

\begin{proof}
Let $\left\{ v_{n}(t),x_{n}(t)\right\} _{t\in[0,h]}$ be the solution
of the exact underdamped Langevin diffusion starting from $\left(v_{n},x_{n}\right)$
. By definition, for $t\in[0,h]$,
\begin{eqnarray*}
\frac{\d f(x_{n}(t))}{\d t} & = & \nabla f(x_{n}(t))^{T}\frac{\d x_{n}(t)}{\d t}\\
 & = & \nabla f(x_{n}(t))^{T}v_{n}(t),
\end{eqnarray*}
so
\begin{eqnarray}
f(x_{n}(h)) & = & f(x_{n}(0))+\int_{0}^{h}\d f(x_{n}(t))\nonumber \\
 & = & f(x_{n}(0))+\int_{0}^{h}\nabla f(x_{n}(t))^{T}v_{n}(t)\d t.\label{eq:sum_v_2}
\end{eqnarray}

Also, since
\begin{eqnarray*}
\d v_{n}(t) & = & \left(-2v_{n}(t)-u\nabla f(x_{n}(t))\right)\d t+2\sqrt{u}\d B_{t},
\end{eqnarray*}
by Ito's lemma,
\begin{eqnarray*}
\d\frac{1}{2}\left\Vert v_{n}(t)\right\Vert ^{2} & = & \left\langle v_{n}(t),2\sqrt{u}\d B_{t}\right\rangle +\left(\left\langle v_{n}(t),-2v_{n}(t)-u\nabla f(x_{n}(t))\right\rangle +\frac{1}{2}\cdot4u\text{Tr}(I_{d})\right)\d t\\
 & = & 2\sqrt{u}v_{n}(t)^{T}\d B_{t}+\left(-2\left\Vert v_{n}(t)\right\Vert ^{2}-uv_{n}(t)^{T}\nabla f(x_{n}(t))+2ud\right)\d t,
\end{eqnarray*}
and therefore
\begin{equation}
\E\frac{1}{2u}\left\Vert v_{n}(h)\right\Vert ^{2}=\E\frac{1}{2u}\left\Vert v_{n}(0)\right\Vert ^{2}+\E\int_{0}^{h}\left(4d-\frac{2}{u}\left\Vert v_{n}(t)\right\Vert ^{2}-v_{n}(t)^{T}\nabla f(x_{n}(t))+2d\right)\d t.\label{eq:sum_v_1}
\end{equation}

Now, we consider the term $\frac{1}{2u}\left\Vert v_{n}(h)\right\Vert ^{2}+f(x_{n}(h)).$
By $\text{\ensuremath{\left(\ref{eq:sum_v_2}\right)}}$and $\text{\ensuremath{\left(\ref{eq:sum_v_1}\right)}}$,
\begin{eqnarray*}
 &  & \E\left[\frac{1}{2u}\left\Vert v_{n}(h)\right\Vert ^{2}+f(x_{n}(h))\right]\\
 & = & \E\left[\frac{1}{2u}\left\Vert v_{n}(0)\right\Vert ^{2}+f(x_{n}(0))+\int_{0}^{h}\left(-\frac{2}{u}\left\Vert v_{n}(t)\right\Vert ^{2}+6d\right)\d t\right]\\
 & \text{\ensuremath{\leq}} & \E\left[\frac{1}{2u}\left\Vert v_{n}(0)\right\Vert ^{2}+f(x_{n}(0))-\frac{2}{u}h\inf_{t\in[0,h]}\left\Vert v_{n}(t)\right\Vert ^{2}+6dh\right]\\
 & \leq & \E\left[\frac{1}{2u}\left\Vert v_{n}(0)\right\Vert ^{2}+f(x_{n}(0))\right]-\frac{2}{3}hL\E\left\Vert v_{n}(0)\right\Vert ^{2}+O\left(uh^{3}\E\left\Vert \nabla f(x_{n}(0))\right\Vert ^{2}+dh\right),
\end{eqnarray*}
where the first step follows by $\text{\ensuremath{\left(\ref{eq:sum_v_2}\right)}}$
and $\text{\ensuremath{\left(\ref{eq:sum_v_1}\right)}}$ and the third
step follows by Lemma $\text{\ref{lem:idealULD}}$.

Since
\begin{eqnarray*}
 &  & \E\left[\left\Vert v_{n+1}(0)\right\Vert ^{2}-\left\Vert v_{n}(h)\right\Vert ^{2}\right]\\
 & = & \E\left(v_{n+1}(0)-v_{n}(h)\right)^{T}\left(v_{n+1}(0)+v_{n}(h)\right)\\
 & \leq & \frac{1}{h^{2}}\E\left\Vert v_{n+1}(0)-v_{n}(h)\right\Vert ^{2}+\frac{1}{2}h^{2}\E\left\Vert v_{n+1}(0)+v_{n}(h)\right\Vert ^{2}\\
 & \leq & \frac{1}{h^{2}}\E\left\Vert v_{n+1}(0)-v_{n}(h)\right\Vert ^{2}+h^{2}\E\left\Vert v_{n+1}(0)-v_{n}(h)\right\Vert ^{2}+4h^{2}\E\left\Vert v_{n}(h)\right\Vert ^{2}\\
 & \leq & \frac{2}{h^{2}}\E\left\Vert v_{n+1}(0)-v_{n}(h)\right\Vert ^{2}+4h^{2}\E\left\Vert v_{n}(h)\right\Vert ^{2}\\
 & \leq & O\left(h^{2}\E\left\Vert v_{n}(0)\right\Vert ^{2}+u^{2}h^{2}\E\left\Vert \nabla f(x_{n}(0))\right\Vert ^{2}+udh^{3}\right),
\end{eqnarray*}
where the first inequality follows by the inequality $2ab\leq a^{2}+b^{2},$
the second inequality follows by Young's inequality and the last inequality
follows by Lemma $\text{\ref{lem:error}}$ and Lemma $\text{\ref{lem:idealULD}}$.

Since
\begin{eqnarray*}
\E f(x_{n+1}(0))-f(x_{n}(h)) & \leq & O\left(uh^{3}\E\left\Vert \nabla f(x_{n}(0))\right\Vert ^{2}+Lh^{5}\E\left\Vert v_{n}(0)\right\Vert ^{2}+dh^{6}\right),
\end{eqnarray*}
which is shown in Lemma $\text{\ref{lem:f_error}},$ we have
\begin{eqnarray*}
 &  & \E\left[\frac{1}{2u}\left\Vert v_{n+1}(0)\right\Vert ^{2}+f(x_{n+1}(0))\right]\\
 & \leq & \E\left[\frac{1}{2u}\left\Vert v_{n}(0)\right\Vert ^{2}+f(x_{n}(0))\right]-\frac{2}{3}hL\E\left\Vert v_{n}(0)\right\Vert ^{2}+O\left(uh^{3}\E\left\Vert \nabla f(x_{n}(0))\right\Vert ^{2}+dh\right)\\
 &  & +O\left(h^{2}L\E\left\Vert v_{n}(0)\right\Vert ^{2}+uh^{2}\E\left\Vert \nabla f(x_{n}(0))\right\Vert ^{2}+dh^{3}\right)\\
 &  & +O\left(uh^{3}\E\left\Vert \nabla f(x_{n}(0))\right\Vert ^{2}+Lh^{5}\E\left\Vert v_{n}(0)\right\Vert ^{2}+dh^{6}\right)\\
 & \leq & \E\left[\frac{1}{2u}\left\Vert v_{n}(0)\right\Vert ^{2}+f(x_{n}(0))\right]-\frac{1}{3}hL\E\left\Vert v_{n}(0)\right\Vert ^{2}+O\left(uh^{2}\E\left\Vert \nabla f(x_{n}(0))\right\Vert ^{2}+hd\right),
\end{eqnarray*}
where the last step follows by $h$ is small. Summing $n$ from $0$
to $N-1$, we get 
\begin{eqnarray*}
 &  & \sum_{n=0}^{N-1}\E\left[\frac{1}{2u}\left\Vert v_{n+1}(0)\right\Vert ^{2}+f(x_{n+1}(0))\right]\\
 & \leq & \sum_{n=0}^{N-1}\E\left[\frac{1}{2u}\left\Vert v_{n}(0)\right\Vert ^{2}+f(x_{n}(0))\right]-\frac{1}{3}hL\sum_{n=0}^{N-1}\E\left\Vert v_{n}(0)\right\Vert ^{2}\\
 &  & +O\left(uh^{2}\sum_{n=0}^{N-1}\E\left\Vert \nabla f(x_{n}(0))\right\Vert ^{2}+Nhd\right).
\end{eqnarray*}
Since $\left\Vert v_{0}(0)\right\Vert =0$ and $f(x_{0}(0))\leq f(x_{N}(0)),$
\begin{eqnarray*}
\frac{1}{3}hL\sum_{n=0}^{N-1}\E\left\Vert v_{n}(0)\right\Vert ^{2} & \leq & O\left(uh^{2}\sum_{n=0}^{N-1}\E\left\Vert \nabla f(x_{n}(0))\right\Vert ^{2}+Nhd\right),
\end{eqnarray*}
 which implies
\begin{eqnarray*}
\sum_{n=0}^{N-1}\E\left\Vert v_{n}(0)\right\Vert ^{2} & \leq & O\left(u^{2}h\sum_{n=0}^{N-1}\E\left\Vert \nabla f(x_{n}(0))\right\Vert ^{2}+Nud\right).
\end{eqnarray*}
\end{proof}
\begin{lem}
\label{lem:sum_gradf} Assume $h$ is smaller than some given constant.
For each iteration $n=0,...,N-1$, let $\left(v_{n},x_{n}\right)$
be the starting point of Algorithm $\text{\ref{al:lan}}$ in iteration
$n$. Then, the $x_{n}$ in iteration $n=0,...,N-1$ satisfies
\begin{eqnarray*}
\sum_{n=0}^{N-1}\E\left\Vert \nabla f(x_{n})\right\Vert ^{2} & \leq & O\left(NLd+\frac{L}{h}\left|\E\nabla f(x_{N})^{T}v_{N}\right|\right).
\end{eqnarray*}
Furthermore, the $v_{n}$ in iteration $n=0,...,N-1$ satisfies
\begin{eqnarray*}
\sum_{n=0}^{N-1}\E\left\Vert v_{n}\right\Vert ^{2} & \leq & O\left(Nud+u\left|\E\nabla f(x_{N})^{T}v_{N}\right|\right).
\end{eqnarray*}
\end{lem}

\begin{proof}
For each iteration $n=0,...,N-1$, let $\left\{ v_{n}(t),x_{n}(t)\right\} _{t\in[0,h]}$
be the exact underdamped Langevin diffusion starting from $\left(v_{n},x_{n}\right)$
computed in Algorithm $\text{\ref{al:lan}}$. By definition,
\begin{eqnarray*}
 &  & \E\left[\d\nabla f(x_{n}(t))^{T}v_{n}(t)\right]\\
 & = & \E\left[v_{n}(t)^{T}\nabla^{2}f(x_{n}(t))v_{n}(t)+\nabla f(x_{n}(t))^{T}\d v_{n}(t)\right]\\
 & = & \E\left[v_{n}(t)^{T}\nabla^{2}f(x_{n}(t))v_{n}(t)-2\nabla f(x_{n}(t))^{T}v_{n}(t)-u\left\Vert \nabla f(x_{n}(t))\right\Vert ^{2}\right].
\end{eqnarray*}
So we have
\begin{eqnarray}
 &  & \E\left[\nabla f(x_{n}(h))^{T}v_{n}(h)\right]\nonumber \\
 & = & \E\left[\nabla f(x_{n}(0))^{T}v_{n}(0)+\int_{0}^{h}\d\nabla f(x_{n}(t))^{T}v_{n}(t)\right]\nonumber \\
 & = & \E\Bigg[\nabla f(x_{n}(0))^{T}v_{n}(0)+\int_{0}^{h}v_{n}(t)^{T}\nabla^{2}f(x_{n}(t))v_{n}(t)-2\nabla f(x_{n}(t))^{T}v_{n}(t)\nonumber \\
 &  & -u\left\Vert \nabla f(x_{n}(t))\right\Vert ^{2}\d t\Bigg]\nonumber \\
 & \text{\ensuremath{\leq}} & \E\left[\nabla f(x_{n}(0))^{T}v_{n}(0)+3L\int_{0}^{h}\left\Vert v_{n}(t)\right\Vert ^{2}\d t-\frac{1}{2}\int_{0}^{h}u\left\Vert \nabla f(x_{n}(t))\right\Vert ^{2}\d t\right]\nonumber \\
 & \leq & \E\left[\nabla f(x_{n}(0))^{T}v_{n}(0)+3Lh\sup_{t\in[0,h]}\left\Vert v_{n}(t)\right\Vert ^{2}-\frac{1}{2}hu\inf_{t\in[0,h]}\left\Vert \nabla f(x_{n}(t))\right\Vert ^{2}\right]\nonumber \\
 & \leq & \E\nabla f(x_{n}(0))^{T}v_{n}(0)-\frac{1}{6}hu\E\left\Vert \nabla f(x_{n}(0))\right\Vert ^{2}+O\left(h^{3}L\E\left\Vert v_{n}(0)\right\Vert ^{2}+dh^{4}\right)\nonumber \\
 &  & +3Lh\cdot O\left(\E\left\Vert v_{n}(0)\right\Vert ^{2}+u^{2}h^{2}\E\left\Vert \nabla f(x_{n}(0))\right\Vert ^{2}+udh\right)\nonumber \\
 & \leq & \E\nabla f(x_{n}(0))^{T}v_{n}(0)-\frac{1}{6}hu\E\left\Vert \nabla f(x_{n}(0))\right\Vert ^{2}\nonumber \\
 &  & +O\left(Lh\E\left\Vert v_{n}(0)\right\Vert ^{2}+uh^{3}\E\left\Vert \nabla f(x_{n}(0))\right\Vert ^{2}+dh^{2}\right),\label{eq:sum_gradf_1}
\end{eqnarray}
where the third step follows by Young's inequality, the fifth step
follows by Lemma $\text{\ref{lem:idealULD}}$ and the last step follows
by $h$ is small. Also, we have 
\begin{eqnarray}
 &  & \E\left[\nabla f(x_{n+1}(0))^{T}v_{n+1}(0)-\nabla f(x_{n}(h))^{T}v_{n}(h)\right]\nonumber \\
 & = & \E\left(\nabla f(x_{n+1}(0))-\nabla f(x_{n}(h))+\nabla f(x_{n}(h))\right)^{T}\left(v_{n+1}(0)-v_{n}(h)\right)\nonumber \\
 &  & +\E\left(\nabla f(x_{n+1}(0)-\nabla f(x_{n}(h))\right)^{T}v_{n}(h)\nonumber \\
 & \leq & u\E\left\Vert \nabla f(x_{n+1}(0))-\nabla f(x_{n}(h))\right\Vert ^{2}+L\E\left\Vert v_{n+1}(0)-v_{n}(h)\right\Vert ^{2}+uh^{2}\E\left\Vert \nabla f(x_{n}(h))\right\Vert ^{2}\nonumber \\
 &  & +\frac{L}{h^{2}}\E\left\Vert v_{n+1}(0)-v_{n}(h)\right\Vert ^{2}+\frac{u}{h}\E\left\Vert \nabla f(x_{n+1}(0))-\nabla f(x_{n}(h))\right\Vert ^{2}+hL\E\left\Vert v_{n}(h)\right\Vert ^{2}\nonumber \\
 & \leq & \frac{2u}{h}\E\left\Vert \nabla f(x_{n+1}(0))-\nabla f(x_{n}(h))\right\Vert ^{2}+\frac{2L}{h^{2}}\E\left\Vert v_{n+1}(0)-v_{n}(h)\right\Vert ^{2}+uh^{2}\E\left\Vert \nabla f(x_{n}(h))\right\Vert ^{2}\nonumber \\
 &  & +hL\E\left\Vert v_{n}(h)\right\Vert ^{2}\nonumber \\
 & \leq & \frac{2L}{h}\cdot O\left(h^{6}\E\left\Vert v_{n}(0)\right\Vert ^{2}+h^{4}u^{2}\E\left\Vert \nabla f(x_{n}(0))\right\Vert ^{2}+udh^{7}\right)\nonumber \\
 &  & +\frac{2L}{h^{2}}\cdot O\left(h^{4}\E\left\Vert v_{n}(0)\right\Vert ^{2}+u^{2}h^{4}\E\left\Vert \nabla f(x_{n}(0))\right\Vert ^{2}+udh^{5}\right)\nonumber \\
 &  & +uh^{2}\cdot O\left(\E\left\Vert \nabla f(x_{n}(0))\right\Vert ^{2}+L^{2}h^{2}\E\left\Vert v_{n}(0)\right\Vert ^{2}+Ldh^{3}\right)\nonumber \\
 &  & +hL\cdot O\left(\E\left\Vert v_{n}(0)\right\Vert ^{2}+u^{2}h^{2}\E\left\Vert \nabla f(x_{n}(0))\right\Vert ^{2}+udh\right)\nonumber \\
 & \leq & O\left(hL\E\left\Vert v_{n}(0)\right\Vert ^{2}+uh^{2}\E\left\Vert \nabla f(x_{n}(0))\right\Vert ^{2}+dh^{2}\right),\label{eq:sum_gradf_2}
\end{eqnarray}
where the second step follows by Young's inequality and the fourth
step follows by Lemma $\text{\ref{lem:error}}$ and Lemma $\text{\ref{lem:idealULD}}$.
Combining $\text{\ensuremath{\left(\ref{eq:sum_gradf_1}\right)}}$
and $\text{\ensuremath{\left(\ref{eq:sum_gradf_2}\right)}},$
\begin{eqnarray*}
\E\nabla f(x_{n+1}(0))^{T}v_{n+1}(0) & \leq & \E\nabla f(x_{n}(0))^{T}v_{n}(0)-\frac{1}{6}hu\E\left\Vert \nabla f(x_{n}(0))\right\Vert ^{2}\\
 &  & +O\left(Lh\E\left\Vert v_{n}(0)\right\Vert ^{2}+uh^{3}\E\left\Vert \nabla f(x_{n}(0))\right\Vert ^{2}+dh^{2}\right)\\
 &  & +O\left(Lh\E\left\Vert v_{n}(0)\right\Vert ^{2}+uh^{2}\E\left\Vert \nabla f(x_{n}(0))\right\Vert ^{2}+dh^{2}\right)\\
 & \leq & \E\nabla f(x_{n}(0))^{T}v_{n}(0)-\frac{1}{6}hu\E\left\Vert \nabla f(x_{n}(0))\right\Vert ^{2}\\
 &  & +O\left(Lh\E\left\Vert v_{n}(0)\right\Vert ^{2}+uh^{2}\E\left\Vert \nabla f(x_{n}(0))\right\Vert ^{2}+dh^{2}\right).
\end{eqnarray*}

Summing from $n=0$ to $N-1$,
\begin{align*}
\sum_{n=0}^{N-1}\E\nabla f(x_{n+1}(0))^{T}v_{n+1}(0) & \leq\sum_{n=0}^{N-1}\E\nabla f(x_{n}(0))^{T}v_{n}(0)-\frac{1}{6}hu\sum_{n=0}^{N-1}\E\left\Vert \nabla f(x_{n}(0))\right\Vert ^{2}\\
 & +O\left(Lh\sum_{n=0}^{N-1}\E\left\Vert v_{n}(0)\right\Vert ^{2}+uh^{2}\sum_{n=0}^{N-1}\E\left\Vert \nabla f(x_{n}(0))\right\Vert ^{2}+Ndh^{2}\right)\\
 & \leq\sum_{n=0}^{N-1}\E\nabla f(x_{n}(0))^{T}v_{n}(0)-\frac{1}{6}hu\sum_{n=0}^{N-1}\left\Vert \nabla f(x_{n}(0))\right\Vert ^{2}\\
 & +O\left(Lh\left(u^{2}h\sum_{n=0}^{N-1}\E\left\Vert \nabla f(x_{n}(0))\right\Vert ^{2}+Nud\right)+Ndh^{2}\right)\\
 & \leq\sum_{n=0}^{N-1}\E\nabla f(x_{n}(0))^{T}v_{n}(0)-\frac{1}{8}hu\sum_{n=0}^{N-1}\left\Vert \nabla f(x_{n}(0))\right\Vert ^{2}+O\left(Ndh\right),
\end{align*}
where the second step follows by Lemma $\text{\ref{lem:sum_v}}$ and
the last step follows by $h$ is small. Then, since $v_{0}=0$,
\begin{eqnarray*}
\frac{1}{8}hu\sum_{n=0}^{N-1}\E\left\Vert \nabla f(x_{n}(0))\right\Vert ^{2} & \leq & O\left(Ndh+\left|\E\nabla f(x_{N}(0))^{T}v_{N}(0)\right|\right),
\end{eqnarray*}
 which implies
\begin{eqnarray*}
\sum_{n=0}^{N-1}\E\left\Vert \nabla f(x_{n}(0))\right\Vert ^{2} & \leq & O\left(NLd+\frac{L}{h}\left|\E\nabla f(x_{N}(0))^{T}v_{N}(0)\right|\right).
\end{eqnarray*}
 By Lemma $\text{\ref{lem:sum_v}}$,
\begin{eqnarray*}
\sum_{n=0}^{N-1}\E\left\Vert v_{n}(0)\right\Vert ^{2} & \leq & O\left(u^{2}h\sum_{n=0}^{N-1}\E\left\Vert \nabla f(x_{n}(0))\right\Vert ^{2}+Nud\right)\\
 & \leq & O\left(Nud+u\left|\E\nabla f(x_{N}(0))^{T}v_{N}(0)\right|\right).
\end{eqnarray*}
\end{proof}

\section{Proof of Theorem $\text{\ref{thm:main}}$}

\label{sec:Proof-of-Theorem}

Here, we combine Lemma $\text{\ref{lem:sum_gradf}}$ and Lemma $\text{\ref{lem:error}}$
to prove our main result.
\begin{proof}
Let $x_{n+\frac{1}{2}}$, $x_{n}$ and $v_{n}$ be the iterates of
Algorithm $\ref{al:lan}$. Let $\left(y_{n},w_{n}\right)$ be the
$n$-th step of the exact underdamped Langevin diffusion, starting
from a random point $\left(y_{0},w_{0}\right)\propto\exp\left(-\left(f(y)+\frac{L}{2}\left\Vert w\right\Vert ^{2}\right)\right)$,
coupled with $\left(x_{n},v_{n}\right)$ through the same Brownian
motion. Let $\left(x_{n+1}^{*},v_{n+1}^{*}\right)$ be the 1-step
exact Langevin diffusion starting from $\left(x_{n},v_{n}\right).$
For any iteration $n$, let $\E_{\alpha}$ be the expectation taken
over the random choice of $\alpha$ in iteration $n$. Then,
\begin{eqnarray*}
 &  & \E_{\alpha}\left[\left\Vert x_{n}-y_{n}\right\Vert ^{2}+\left\Vert \left(x_{n}+v_{n}\right)-\left(y_{n}+w_{n}\right)\right\Vert ^{2}\right]\\
 & = & \E_{\alpha}\left[\left\Vert \left(x_{n}-x_{n}^{*}\right)-\left(y_{n}-x_{n}^{*}\right)\right\Vert ^{2}+\left\Vert \left(x_{n}+v_{n}-x_{n}^{*}-v_{n}^{*}\right)-\left(y_{n}+w_{n}-x_{n}^{*}-v_{n}^{*}\right)\right\Vert ^{2}\right]\\
 & \leq & \left\Vert y_{n}-x_{n}^{*}\right\Vert ^{2}+\left\Vert y_{n}+w_{n}-x_{n}^{*}-v_{n}^{*}\right\Vert ^{2}+\E_{\alpha}\left\Vert x_{n}-x_{n}^{*}\right\Vert ^{2}+\E_{\alpha}\left\Vert x_{n}+v_{n}-x_{n}^{*}-v_{n}^{*}\right\Vert ^{2}\\
 &  & -2\left(y_{n}-x_{n}^{*}\right)^{T}\left(\E_{\alpha}x_{n}-x_{n}^{*}\right)-2\left(y_{n}+w_{n}-x_{n}^{*}-v_{n}^{*}\right)^{T}\left(\E_{\alpha}\left[x_{n}+v_{n}\right]-x_{n}^{*}-v_{n}^{*}\right)\\
 & \leq & \left(1+\frac{h}{2\kappa}\right)\left(\left\Vert y_{n}-x_{n}^{*}\right\Vert ^{2}+\left\Vert y_{n}+w_{n}-x_{n}^{*}-v_{n}^{*}\right\Vert ^{2}\right)\\
 &  & +\frac{2\kappa}{h}\left(\left\Vert \E_{\alpha}x_{n}-x_{n}^{*}\right\Vert ^{2}+\left\Vert \E_{\alpha}\left[x_{n}+v_{n}\right]-x_{n}^{*}-v_{n}^{*}\right\Vert ^{2}\right)+\E_{\alpha}\left\Vert x_{n}-x_{n}^{*}\right\Vert ^{2}\\
 &  & +\E_{\alpha}\left\Vert x_{n}+v_{n}-x_{n}^{*}-v_{n}^{*}\right\Vert ^{2},
\end{eqnarray*}
where the second step follows by $y_{n},$ $w_{n}$, $x_{n}^{*}$
and $v_{n}^{*}$ are independent of the choice of $\alpha$ and the
third follows by Young's inequality. Then,
\begin{eqnarray*}
 &  & \E\left[\left\Vert x_{N}-y_{N}\right\Vert ^{2}+\left\Vert \left(x_{N}+v_{N}\right)-\left(y_{N}+w_{N}\right)\right\Vert ^{2}\right]\\
 & \leq & \left(1+\frac{h}{2\kappa}\right)e^{-\frac{h}{\kappa}}\E\left[\left\Vert y_{N-1}-x_{N-1}\right\Vert ^{2}+\left\Vert y_{N-1}+w_{N-1}-x_{N-1}-v_{N-1}\right\Vert ^{2}\right]\\
 &  & +\frac{2\kappa}{h}\left(\E\left\Vert \E_{\alpha}x_{N}-x_{N}^{*}\right\Vert ^{2}+\E\left\Vert \E_{\alpha}x_{N}+v_{N}-x_{N}^{*}-v_{N}^{*}\right\Vert ^{2}\right)\\
 &  & +\left(\E\left\Vert x_{N}-x_{N}^{*}\right\Vert ^{2}+\E\left\Vert x_{N}+v_{N}-x_{N}^{*}-v_{N}^{*}\right\Vert ^{2}\right)\\
 & \leq & e^{-\frac{h}{2\kappa}}\E\left[\left\Vert y_{N-1}-x_{N-1}\right\Vert ^{2}+\left\Vert y_{N-1}+w_{N-1}-x_{N-1}-v_{N-1}\right\Vert ^{2}\right]\\
 &  & +\frac{2\kappa}{h}\left(2\E\left\Vert \E_{\alpha}v_{N}-v_{N}^{*}\right\Vert ^{2}+3\E\left\Vert \E_{\alpha}x_{N}-x_{N}^{*}\right\Vert ^{2}\right)+\left(2\E\left\Vert v_{N}-v_{N}^{*}\right\Vert ^{2}+3\E\left\Vert x_{N}-x_{N}^{*}\right\Vert ^{2}\right)\\
 & \leq & e^{-\frac{Nh}{2\kappa}}\E\left[\left\Vert y_{0}-x_{0}\right\Vert ^{2}+\left\Vert y_{0}+w_{0}-x_{0}-v_{0}\right\Vert ^{2}\right]\\
 &  & +\sum_{n=1}^{N}\frac{2\kappa}{h}\left(2\E\left\Vert \E_{\alpha}v_{n}-v_{n}^{*}\right\Vert ^{2}+3\E\left\Vert \E_{\alpha}x_{n}-x_{n}^{*}\right\Vert ^{2}\right)\\
 &  & +\sum_{n=1}^{N}\left(2\E\left\Vert v_{n}-v_{n}^{*}\right\Vert ^{2}+3\E\left\Vert x_{n}-x_{n}^{*}\right\Vert ^{2}\right),
\end{eqnarray*}
where the first step follows by Lemma $\text{\ref{lem:ULD}}$, the
second step follows by $1+\frac{h}{2\kappa}\leq e^{\frac{h}{2\kappa}}$
, and the last step follows by induction.

Since $(y_{N},w_{N})$ follows the distribution $p^{*}\propto\exp\left(-\left(f(y)+\frac{L}{2}\left\Vert w\right\Vert ^{2}\right)\right)$,
$\E\left\Vert w_{N}\right\Vert ^{2}=\frac{d}{L}$. By Proposition
1 of \cite{durmus2016high}, $\E\left\Vert y_{0}-x_{0}\right\Vert ^{2}\leq\frac{d}{m}$.
Then,
\begin{align*}
\E\left[\left\Vert y_{0}-x_{0}\right\Vert ^{2}+\left\Vert y_{0}+w_{0}-x_{0}-v_{0}\right\Vert ^{2}\right] & \leq3\E\left\Vert y_{0}-x_{0}\right\Vert ^{2}+2\E\left\Vert w_{0}-v_{0}\right\Vert ^{2}\\
 & \leq5\frac{d}{m}.
\end{align*}
When $N=\frac{2\kappa}{h}\log\left(\frac{20}{\epsilon^{2}}\right)$,
\begin{eqnarray*}
e^{-\frac{Nh}{2\kappa}}\E\left[\left\Vert y_{0}-x_{0}\right\Vert ^{2}+\left\Vert y_{0}+w_{0}-x_{0}-v_{0}\right\Vert ^{2}\right] & \leq & \frac{\epsilon^{2}d}{4m}.
\end{eqnarray*}
By Lemma $\text{\ref{lem:error}}$,
\begin{eqnarray*}
 &  & \sum_{n=1}^{N}\frac{2\kappa}{h}\left(2\E\left\Vert \E_{\alpha}v_{n}-v_{n}^{*}\right\Vert ^{2}+3\E\left\Vert \E_{\alpha}x_{n}-x_{n}^{*}\right\Vert ^{2}\right)\\
 & \leq & O\left(h^{7}\kappa\sum_{n=0}^{N-1}\E\left\Vert v_{n}\right\Vert ^{2}+\frac{u}{m}h^{9}\sum_{n=0}^{N-1}\E\left\Vert \nabla f(x_{n})\right\Vert ^{2}+\frac{1}{m}Ndh^{8}\right),
\end{eqnarray*}
and
\begin{eqnarray*}
 &  & \sum_{n=1}^{N}\left(2\E\left\Vert v_{n}-v_{n}^{*}\right\Vert ^{2}+3\E\left\Vert x_{n}-x_{n}^{*}\right\Vert ^{2}\right)\\
 & \leq & O\left(h^{4}\sum_{n=0}^{N-1}\E\left\Vert v_{n}\right\Vert ^{2}+u^{2}h^{4}\sum_{n=0}^{N-1}\E\left\Vert \nabla f(x_{n})\right\Vert ^{2}+Nudh^{5}\right).
\end{eqnarray*}
By Lemma 2 of \cite{dalalyan2017further}, $\E\left\Vert \nabla f(y_{N})\right\Vert ^{2}\leq dL$.
Then, by $\E\left\Vert \nabla f(y_{N})\right\Vert ^{2}\leq dL$ and
$\E\left\Vert w_{N}\right\Vert ^{2}=\frac{d}{L}$,
\begin{eqnarray*}
\left|\E\nabla f(x_{N})^{T}v_{N}\right| & \leq & \E\left[L\left\Vert v_{N}\right\Vert ^{2}+u\left\Vert \nabla f(x_{N})\right\Vert ^{2}\right]\\
 & \leq & 2\E\left[L\left\Vert w_{N}\right\Vert ^{2}+L\left\Vert v_{N}-w_{N}\right\Vert ^{2}+u\left\Vert \nabla f(y_{N})\right\Vert ^{2}+L\left\Vert x_{N}-y_{N}\right\Vert ^{2}\right]\\
 & \leq & 4d+2L\E\left[\left\Vert v_{N}-w_{N}\right\Vert ^{2}+\left\Vert x_{N}-y_{N}\right\Vert ^{2}\right]\\
 & \leq & 4d+6L\E\left[\left\Vert x_{N}-y_{N}\right\Vert ^{2}+\left\Vert \left(x_{N}+v_{N}\right)-\left(y_{N}+w_{N}\right)\right\Vert ^{2}\right],
\end{eqnarray*}

By Lemma $\text{\ref{lem:sum_gradf}}$ and our choice of $N$,
\[
\sum_{n=0}^{N-1}\left\Vert \nabla f(x_{n}(0))\right\Vert ^{2}\leq O\left(\frac{\kappa dL}{h}\log\left(\frac{1}{\epsilon^{2}}\right)+\frac{L^{2}}{h}\E\left[\left\Vert x_{N}-y_{N}\right\Vert ^{2}+\left\Vert \left(x_{N}+v_{N}\right)-\left(y_{N}+w_{N}\right)\right\Vert ^{2}\right]\right),
\]
and
\[
\sum_{n=0}^{N-1}\E\left\Vert v_{n}(0)\right\Vert ^{2}\leq O\left(\frac{d}{hm}\log\left(\frac{1}{\epsilon^{2}}\right)+\E\left[\left\Vert x_{N}-y_{N}\right\Vert ^{2}+\left\Vert \left(x_{N}+v_{N}\right)-\left(y_{N}+w_{N}\right)\right\Vert ^{2}\right]\right).
\]
Thus, 
\begin{align*}
 & \sum_{n=1}^{N}\frac{2\kappa}{h}\left(2\E\left\Vert \E_{\alpha}v_{n}-v_{n}^{*}\right\Vert ^{2}+3\E\left\Vert \E_{\alpha}x_{n}-x_{n}^{*}\right\Vert ^{2}\right)+\sum_{n=1}^{N}\left(2\E\left\Vert v_{n}-v_{n}^{*}\right\Vert ^{2}+3\E\left\Vert x_{n}-x_{n}^{*}\right\Vert ^{2}\right)\\
 & \leq O\left(\left(\frac{\kappa dh^{6}}{m}+\frac{dh^{3}}{m}\right)\log\left(\frac{1}{\epsilon^{2}}\right)\right)\\
 & +O\left(\kappa h^{7}+h^{3}\right)\E\left[\left\Vert x_{N}-y_{N}\right\Vert ^{2}+\left\Vert \left(x_{N}+v_{N}\right)-\left(y_{N}+w_{N}\right)\right\Vert ^{2}\right].
\end{align*}

Then, we can choose a small constant $C$ such that if we let 
\[
h=C\min\left(\frac{\epsilon^{1/3}}{\kappa^{1/6}}\log^{-1/6}\left(\frac{1}{\epsilon^{2}}\right),\epsilon^{2/3}\log^{-1/3}\left(\frac{1}{\epsilon^{2}}\right)\right),
\]
then
\begin{align*}
 & \sum_{n=1}^{N}\frac{2\kappa}{h}\left(2\E\left\Vert \E_{\alpha}v_{n}-v_{n}^{*}\right\Vert ^{2}+3\E\left\Vert \E_{\alpha}x_{n}-x_{n}^{*}\right\Vert ^{2}\right)+\sum_{n=1}^{N}\left(2\E\left\Vert v_{n}-v_{n}^{*}\right\Vert ^{2}+3\E\left\Vert x_{n}-x_{n}^{*}\right\Vert ^{2}\right)\\
 & \leq\frac{\epsilon^{2}d}{4m}+\frac{1}{2}\E\left[\left\Vert x_{N}-y_{N}\right\Vert ^{2}+\left\Vert \left(x_{N}+v_{N}\right)-\left(y_{N}+w_{N}\right)\right\Vert ^{2}\right].
\end{align*}
Therefore,
\begin{eqnarray*}
 &  & \E\left[\left\Vert x_{N}-y_{N}\right\Vert ^{2}+\left\Vert \left(x_{N}+v_{N}\right)-\left(y_{N}+w_{N}\right)\right\Vert ^{2}\right]\\
 & \leq & \frac{\epsilon^{2}d}{4m}+\frac{\epsilon^{2}d}{4m}+\frac{1}{2}\E\left[\left\Vert x_{N}-y_{N}\right\Vert ^{2}+\left\Vert \left(x_{N}+v_{N}\right)-\left(y_{N}+w_{N}\right)\right\Vert ^{2}\right]\\
 & = & \frac{\epsilon^{2}d}{2m}+\frac{1}{2}\E\left[\left\Vert x_{N}-y_{N}\right\Vert ^{2}+\left\Vert \left(x_{N}+v_{N}\right)-\left(y_{N}+w_{N}\right)\right\Vert ^{2}\right],
\end{eqnarray*}
which implies
\[
\E\left[\left\Vert x_{N}-y_{N}\right\Vert ^{2}\right]\leq\E\left[\left\Vert x_{N}-y_{N}\right\Vert ^{2}+\left\Vert \left(x_{N}+v_{N}\right)-\left(y_{N}+w_{N}\right)\right\Vert ^{2}\right]\leq\frac{\epsilon^{2}d}{m}.
\]
By our choice of $h$,
\begin{eqnarray*}
N & \leq & \tilde{O}\left(\frac{\kappa^{7/6}}{\epsilon^{1/3}}+\frac{\kappa}{\epsilon^{2/3}}\right).
\end{eqnarray*}
\end{proof}

\section{Discretization Error of Algorithm $\text{\ref{al:lan2}}$\label{sec:Discretization-Error-of}}

Here, we bound the discretization error in one step of Algorithm $\text{\ref{al:lan2}}$.
Since the terms $\E\left\Vert \E_{\alpha}x_{n+1}-x_{n}^{*}(h)\right\Vert ^{2}$
and $\E\left\Vert x_{n+1}-x_{n}^{*}(h)\right\Vert ^{2}$ are dominated
by the terms $\E\left\Vert \E_{\alpha}v_{n+1}-v_{n}^{*}(h)\right\Vert ^{2}$
and $\E\left\Vert v_{n+1}-v_{n}^{*}(h)\right\Vert ^{2}$, we bound
only the later two terms.
\begin{lem}
\label{lem:step_induc}Assume that $R^{4}\delta^{4}\leq\frac{1}{4}$.
Let $x_{n}^{(k-1,i)}$ for $i=1,...,R$, $k=1,...,K$ be the intermediate
value computed in iteration $n$ of Algorithm $\text{\ref{al:lan2}}$.
Let $\left\{ x_{n}^{*}(t),v_{n}^{*}(t)\right\} _{t\in[0,h]}$ be the
ideal underdamped Langevin diffusion, starting from $x_{n}^{*}(0)=x_{n}$
and $v_{n}^{*}(0)=v_{n}$, coupled through a shared Brownian motion
with $\left\{ x_{n}^{(k-1,i)}\right\} _{i=1,...,R,k=1,...,K}.$ Then,
for any $i=1,...,R,$ and $k=1,...,K-1$, 
\begin{eqnarray*}
\E\left\Vert x_{n}^{(k,i)}-x_{n}^{*}(\alpha_{i}h)\right\Vert ^{2} & \leq & \left(2R^{4}\delta^{4}\right)^{k}\frac{1}{R}\sum_{j=1}^{R}\E\left\Vert x_{n}-x_{n}^{*}(\alpha_{j}h)\right\Vert ^{2}\\
 &  & +4R^{3}\delta^{4}\sum_{j=1}^{R}\E\sup_{s\in[(j-1)\delta,j\delta]}\left\Vert x_{n}^{*}(\alpha_{j}h)-x_{n}^{*}(s)\right\Vert ^{2}.
\end{eqnarray*}
\end{lem}

\begin{proof}
For any $i=1,...,R,$ and $k=1,...,K-1$,
\begin{eqnarray*}
 &  & \E\left\Vert x_{n}^{(k,i)}-x_{n}^{*}(\alpha_{i}h)\right\Vert ^{2}\\
 & \leq & \E\Bigg\|\frac{1}{2}u\sum_{j=1}^{i}\left[\int_{(j-1)\delta}^{\min(j\delta,\alpha_{i}h)}\left(1-e^{-2(\alpha_{i}h-s)}\right)\d s\cdot\nabla f(x_{n}^{(k-1,j)})\right]\\
 &  & -\frac{1}{2}u\int_{0}^{\alpha_{i}h}\left(1-e^{-2(\alpha_{i}h-s)}\right)\nabla f(x_{n}^{*}(s))\d s\Bigg\|^{2}\\
 & \leq & \frac{1}{2}\E\left\Vert u\sum_{j=1}^{i}\left[\int_{(j-1)\delta}^{\min(j\delta,\alpha_{i}h)}\left(1-e^{-2(\alpha_{i}h-s)}\right)\d s\cdot\left(\nabla f(x_{n}^{(k-1,j)})-\nabla f(x_{n}^{*}(\alpha_{j}h))\right)\right]\right\Vert ^{2}\\
 &  & +\frac{1}{2}\E\left\Vert u\sum_{j=1}^{i}\left[\int_{(j-1)\delta}^{\min(j\delta,\alpha_{i}h)}\left(1-e^{-2(\alpha_{i}h-s)}\right)\left(\nabla f(x_{n}^{*}(\alpha_{j}h))-\nabla f(x_{n}^{*}(s))\right)\d s\right]\right\Vert ^{2},
\end{eqnarray*}
where the first step follows by the definition, and the second step
follows by Young's inequality. 

To compute the first term,
\begin{eqnarray}
 &  & \frac{1}{2}\E\left\Vert u\sum_{j=1}^{i}\left[\int_{(j-1)\delta}^{\min(j\delta,\alpha_{i}h)}\left(1-e^{-2(\alpha_{i}h-s)}\right)\d s\cdot\left(\nabla f(x_{n}^{(k-1,j)})-\nabla f(x_{n}^{*}(\alpha_{j}h))\right)\right]\right\Vert ^{2}\nonumber \\
 & \leq & \frac{1}{2}u^{2}R\sum_{j=1}^{i}\E\left\Vert \int_{(j-1)\delta}^{\min(j\delta,\alpha_{i}h)}\left(1-e^{-2(\alpha_{i}h-s)}\right)\d s\cdot\left(\nabla f(x_{n}^{(k-1,j)})-\nabla f(x_{n}^{*}(\alpha_{j}h))\right)\right\Vert ^{2}\nonumber \\
 & \leq & 2R^{3}\delta^{4}\sum_{j=1}^{R}\E\left\Vert x_{n}^{(k-1,j)}-x_{n}^{*}(\alpha_{j}h)\right\Vert ^{2},\label{eq:step_induc_1}
\end{eqnarray}
where the first step follows by the inequality $\left(\sum_{i=1}^{n}a_{i}\right)^{2}\leq n\sum_{i=1}^{n}a_{i}^{2}$,
the second step follows by $1-e^{-2(\alpha_{i}h-s)}\leq2R\delta$
and $\nabla f$ is $L$-Lipschitz. 

For the second term,
\begin{eqnarray}
 &  & \frac{1}{2}\E\left\Vert u\sum_{j=1}^{i}\left[\int_{(j-1)\delta}^{\min(j\delta,\alpha_{i}h)}\left(1-e^{-2(\alpha_{i}h-s)}\right)\left(\nabla f(x_{n}^{*}(\alpha_{j}h))-\nabla f(x_{n}^{*}(s))\right)\d s\right]\right\Vert ^{2}\nonumber \\
 & \leq & \frac{1}{2}u^{2}R\sum_{j=1}^{i}\E\left\Vert \int_{(j-1)\delta}^{\min(j\delta,\alpha_{i}h)}\left(1-e^{-2(\alpha_{i}h-s)}\right)\left(\nabla f(x_{n}^{*}(\alpha_{j}h))-\nabla f(x_{n}^{*}(s))\right)\d s\right\Vert ^{2}\nonumber \\
 & \leq & 2R^{3}\delta^{4}\sum_{j=1}^{R}\E\sup_{s\in[(j-1)\delta,j\delta]}\left\Vert x_{n}^{*}(\alpha_{j}h)-x_{n}^{*}(s)\right\Vert ^{2},\label{eq:step_induc_2}
\end{eqnarray}
where the first step follows by the inequality $\left(\sum_{i=1}^{n}a_{i}\right)^{2}\leq n\sum_{i=1}^{n}a_{i}^{2}$
and the second step follows by $1-e^{-2(\alpha_{i}h-s)}\leq2R\delta$
and $\nabla f$ is $L$-Lipschitz. Thus,
\begin{eqnarray*}
 &  & \E\left\Vert x_{n}^{(k,i)}-x_{n}^{*}(\alpha_{i}h)\right\Vert ^{2}\\
 & \leq & 2R^{3}\delta^{4}\sum_{j=1}^{R}\E\left\Vert x_{n}^{(k-1,j)}-x_{n}^{*}(\alpha_{j}h)\right\Vert ^{2}+2R^{3}\delta^{4}\sum_{j=1}^{R}\E\sup_{s\in[(j-1)\delta,j\delta]}\left\Vert x_{n}^{*}(\alpha_{j}h)-x_{n}^{*}(s)\right\Vert ^{2}\\
 & \leq & \left(2R^{4}\delta^{4}\right)^{k}\frac{1}{R}\sum_{j=1}^{R}\E\left\Vert x_{n}-x_{n}^{*}(\alpha_{j}h)\right\Vert ^{2}\\
 &  & +\left(1+2R^{4}\delta^{4}+...+\left(2R^{4}\delta^{4}\right)^{k-1}\right)2R^{3}\delta^{4}\sum_{j=1}^{R}\E\sup_{s\in[(j-1)\delta,j\delta]}\left\Vert x_{n}^{*}(\alpha_{j}h)-x_{n}^{*}(s)\right\Vert ^{2}\\
 & \leq & \left(2R^{4}\delta^{4}\right)^{k}\frac{1}{R}\sum_{j=1}^{R}\E\left\Vert x_{n}-x_{n}^{*}(\alpha_{j}h)\right\Vert ^{2}+4R^{3}\delta^{4}\sum_{j=1}^{R}\E\sup_{s\in[(j-1)\delta,j\delta]}\left\Vert x_{n}^{*}(\alpha_{j}h)-x_{n}^{*}(s)\right\Vert ^{2},
\end{eqnarray*}
where the first step follows by $\left(\text{\ref{eq:step_induc_1}}\right)$
and $\left(\text{\ref{eq:step_induc_2}}\right)$, the second step
follows by induction, and the third step follows by $2R^{4}\delta^{4}\leq\frac{1}{2}.$
\end{proof}
\begin{lem}
Let $\left(v_{n},x_{n}\right)$ be the iterates of iteration $n$.
Let $x_{n}^{(k,i)}$ for $i=1,...,R$, $k=1,...,K-1$ be the intermediate
value computed in iteration $n$ of Algorithm $\text{\ref{al:lan2}}$.
Let $\left\{ x_{n}^{*}(t),v_{n}^{*}(t)\right\} _{t\in[0,h]}$ be the
ideal underdamped Langevin diffusion, starting from $x_{n}^{*}(0)=x_{n}$
and $v_{n}^{*}(0)=v_{n}$, coupled through a shared Brownian motion
with $\left\{ x_{n}^{(k,i)}\right\} _{i=1,...,R,k=1,...,K-1}.$ Assume
that $h=R\delta\leq\frac{1}{10}$ and $K\geq\Omega\left(\log\frac{1}{\delta^{4}}\right)$.
Let $\E_{\alpha}$ be the expectation taken over the choice of $\alpha_{1},...,\alpha_{R}$
in iteration $n$. Let $\E$ be the expectation taken over other randomness
in iteration $n$. Then, 
\begin{eqnarray*}
\E\left\Vert \E_{\alpha}v_{n+1}-v_{n}^{*}(h)\right\Vert ^{2} & \leq & O\left(R^{6}\delta^{8}\left\Vert v_{n}\right\Vert ^{2}+u^{2}R^{6}\delta^{10}\left\Vert \nabla f(x_{n})\right\Vert ^{2}+R^{6}\delta^{9}ud\right),\\
\E\left\Vert v_{n+1}-v_{n}^{*}(h)\right\Vert ^{2} & \leq & O\left(R^{2}\delta^{4}\left\Vert v_{n}\right\Vert ^{2}+u^{2}R^{2}\delta^{4}\left\Vert \nabla f(x_{n})\right\Vert ^{2}+R^{2}\delta^{5}ud\right).
\end{eqnarray*}
\end{lem}

\begin{proof}
To show the first claim,
\begin{eqnarray}
 &  & \E\left\Vert \E_{\alpha}v_{n+1}-v_{n}^{*}(h)\right\Vert ^{2}\nonumber \\
 & \leq & \E\left\Vert \E_{\alpha}u\sum_{i=1}^{R}\delta e^{-2(h-\alpha_{i}h)}\nabla f(x_{n}^{(K-1,i)})-u\int_{0}^{h}e^{-2(h-s)}\nabla f(x_{n}^{*}(s))\d s\right\Vert ^{2}\nonumber \\
 & \leq & 2\E\left\Vert u\sum_{i=1}^{R}\delta e^{-2(h-\alpha_{i}h)}\nabla f(x_{n}^{(K-1,i)})-u\sum_{i=1}^{R}\delta e^{-2(h-\alpha_{i}h)}\nabla f(x_{n}^{*}(\alpha_{i}h))\right\Vert ^{2}\nonumber \\
 &  & +2\E\left\Vert \E_{\alpha}u\sum_{i=1}^{R}\delta e^{-2(h-\alpha_{i}h)}\nabla f(x_{n}^{*}(\alpha_{i}h))-u\int_{0}^{h}e^{-2(h-s)}\nabla f(x_{n}^{*}(s))\d s\right\Vert ^{2}\nonumber \\
 & \leq & 2\delta^{2}R\sum_{i=1}^{R}\E\left\Vert x_{n}^{(K-1,i)}-x_{n}^{*}(\alpha_{i}h)\right\Vert ^{2}+0\nonumber \\
 & \leq & 2\delta^{2}R\left(2R^{4}\delta^{4}\right)^{K-1}\sum_{i=1}^{R}\E\left\Vert x_{n}-x_{n}^{*}(\alpha_{i}h)\right\Vert ^{2}\nonumber \\
 &  & +8R^{5}\delta^{6}\sum_{i=1}^{R}\E\sup_{s\in[(i-1)\delta,i\delta]}\left\Vert x_{n}^{*}(\alpha_{i}h)-x_{n}^{*}(s)\right\Vert ^{2},\label{eq:para_1}
\end{eqnarray}
where the first step follows by the definition, the second step follows
by Young's inequality, and the third step follows by 
\begin{eqnarray*}
\E_{\alpha}\delta e^{-2(h-\alpha_{i}h)}\nabla f(x_{n}^{*}(\alpha_{i}h)) & = & \int_{(i-1)\delta}^{i\delta}e^{-2(h-s)}\nabla f(x_{n}^{*}(s))\d s.
\end{eqnarray*}

To show the second claim,
\begin{eqnarray*}
 &  & \E\left\Vert v_{n+1}-v_{n}^{*}(h)\right\Vert ^{2}\\
 & \leq & \E\left\Vert u\sum_{i=1}^{R}\delta e^{-2(h-\alpha_{i}h)}\nabla f(x_{n}^{(K-1,i)})-u\int_{0}^{h}e^{-2(h-s)}\nabla f(x_{n}^{*}(s))\d s\right\Vert ^{2}\\
 & \leq & 3\E\left\Vert u\sum_{i=1}^{R}\delta e^{-2(h-\alpha_{i}h)}\nabla f(x_{n}^{(K-1,i)})-u\sum_{i=1}^{R}\delta e^{-2(h-\alpha_{i}h)}\nabla f(x_{n}^{*}(\alpha_{i}h))\right\Vert ^{2}\\
 &  & +3\E\left\Vert u\sum_{i=1}^{R}\int_{(i-1)\delta}^{i\delta}e^{-2(h-\alpha_{i}h)}\left(\nabla f(x_{n}^{*}(\alpha_{i}h))-\nabla f(x_{n}^{*}(s))\right)\d s\right\Vert ^{2}\\
 &  & +3\E\left\Vert u\sum_{i=1}^{R}\int_{(i-1)\delta}^{i\delta}\left(e^{-2(h-\alpha_{i}h)}-e^{-2(h-s)}\right)\nabla f(x_{n}^{*}(s))\d s\right\Vert ^{2}.
\end{eqnarray*}
Like the proof of the third claim, the first term satisfies 
\begin{eqnarray*}
 &  & 3\E\left\Vert u\sum_{i=1}^{R}\delta e^{-2(h-\alpha_{i}h)}\nabla f(x_{n}^{(K-1,i)})-u\sum_{i=1}^{R}\delta e^{-2(h-\alpha_{i}h)}\nabla f(x_{n}^{*}(\alpha_{i}h))\right\Vert ^{2}\\
 & \leq & 3\delta^{2}R\left(2R^{4}\delta^{4}\right)^{K-1}\sum_{i=1}^{R}\E\left\Vert x_{n}-x_{n}^{*}(\alpha_{i}h)\right\Vert ^{2}+12R^{5}\delta^{6}\sum_{i=1}^{R}\E\sup_{s\in[(i-1)\delta,i\delta]}\left\Vert x_{n}^{*}(\alpha_{i}h)-x_{n}^{*}(s)\right\Vert ^{2}.
\end{eqnarray*}
The second term satisfies
\begin{eqnarray*}
 &  & 3\E\left\Vert u\sum_{i=1}^{R}\int_{(i-1)\delta}^{i\delta}e^{-2(h-\alpha_{i}h)}\left(\nabla f(x_{n}^{*}(\alpha_{i}h))-\nabla f(x_{n}^{*}(s))\right)\d s\right\Vert ^{2}\\
 & \leq & 3u^{2}R\sum_{i=1}^{R}\E\left\Vert \int_{(i-1)\delta}^{i\delta}e^{-2(h-\alpha_{i}h)}\left(\nabla f(x_{n}^{*}(\alpha_{i}h))-\nabla f(x_{n}^{*}(s))\right)\d s\right\Vert ^{2}\\
 & \leq & 3\delta^{2}R\sum_{i=1}^{R}\E\sup_{s\in[(i-1)\delta,i\delta]}\left\Vert x_{n}^{*}(\alpha_{i}h)-x_{n}^{*}(s)\right\Vert ^{2},
\end{eqnarray*}
where the first step follows by $\left(\sum_{i=1}^{n}a_{i}\right)^{2}\leq n\sum_{i=1}^{n}a_{i}^{2}$,
and the second step follows by $\nabla f$ is $L$-Lipschitz.

The last term satisfies 
\[
3\E\left\Vert u\sum_{i=1}^{R}\int_{(i-1)\delta}^{i\delta}\left(e^{-2(h-\alpha_{i}h)}-e^{-2(h-s)}\right)\nabla f(x_{n}^{*}(s))\d s\right\Vert ^{2}\leq12u^{2}R^{2}\delta^{4}\E\sup_{s\in[0,h]}\left\Vert \nabla f(x_{n}^{*}(s))\right\Vert ^{2},
\]
which follows by $e^{-2(h-\alpha_{i}h)}-e^{-2(h-s)}\leq2\delta$ for
$s\in[(i-1)\delta,i\delta]$. Thus,
\begin{eqnarray}
 &  & \E\left\Vert v_{n+1}-v_{n}^{*}(h)\right\Vert ^{2}\nonumber \\
 & \leq & 3\delta^{2}R\left(2R^{4}\delta^{4}\right)^{K-1}\sum_{i=1}^{R}\E\left\Vert x_{n}-x_{n}^{*}(\alpha_{i}h)\right\Vert ^{2}\nonumber \\
 &  & +12R^{5}\delta^{6}\sum_{i=1}^{R}\E\sup_{s\in[(i-1)\delta,i\delta]}\left\Vert x_{n}^{*}(\alpha_{i}h)-x_{n}^{*}(s)\right\Vert ^{2}\nonumber \\
 &  & +3\delta^{2}R\sum_{i=1}^{R}\E\sup_{s\in[(i-1)\delta,i\delta]}\left\Vert x_{n}^{*}(\alpha_{i}h)-x_{n}^{*}(s)\right\Vert ^{2}+12u^{2}R^{2}\delta^{4}\E\sup_{s\in[0,h]}\left\Vert \nabla f(x_{n}^{*}(s))\right\Vert ^{2}.\label{eq:para_2}
\end{eqnarray}

By Lemma $\text{\ref{lem:idealULD}}$, for $i=1,...,R$,
\begin{eqnarray*}
\E\left\Vert x_{n}-x_{n}^{*}(\alpha_{i}h)\right\Vert ^{2} & \leq & O\left(R^{2}\delta^{2}\left\Vert v_{n}\right\Vert ^{2}+u^{2}R^{4}\delta^{4}\left\Vert \nabla f(x_{n})\right\Vert ^{2}+udR^{3}\delta^{3}\right),
\end{eqnarray*}
,and
\[
\E\sup_{s\in[(i-1)\delta,i\delta]}\left\Vert x_{n}^{*}(\alpha_{i}h)-x_{n}^{*}(s)\right\Vert ^{2}\leq O\left(\delta^{2}\left\Vert v_{n}\right\Vert ^{2}+u^{2}\delta^{4}\left\Vert \nabla f(x_{n})\right\Vert ^{2}+ud\delta^{3}\right).
\]
Thus, when $K\geq\Omega\left(\log\frac{1}{\delta^{4}}\right)$, since
$R\delta\leq\frac{1}{10}$, $\left(2R^{4}\delta^{4}\right)^{K-1}\leq O\left(\delta^{4}\right)$.
By (\ref{eq:para_1}) and (\ref{eq:para_2}),
\begin{eqnarray*}
\E\left\Vert \E_{\alpha}v_{n+1}-v_{n}^{*}(h)\right\Vert ^{2} & \leq & O\left(R^{6}\delta^{8}\left\Vert v_{n}\right\Vert ^{2}+u^{2}R^{6}\delta^{10}\left\Vert \nabla f(x_{n})\right\Vert ^{2}+R^{6}\delta^{9}ud\right),
\end{eqnarray*}
and
\begin{eqnarray*}
\E\left\Vert v_{n+1}-v_{n}^{*}(h)\right\Vert ^{2} & \leq & O\left(R^{2}\delta^{4}\left\Vert v_{n}\right\Vert ^{2}+u^{2}R^{2}\delta^{4}\E\left\Vert \nabla f(x_{n})\right\Vert ^{2}+R^{2}\delta^{5}ud\right).
\end{eqnarray*}
\end{proof}

\end{document}